\documentclass[journal,twocolumn,10pt]{IEEEtran}

\usepackage{microtype}
\usepackage{graphicx}
\usepackage{subfigure}
\usepackage{booktabs}
\usepackage{microtype}
\usepackage{graphicx}
\usepackage{subfigure}
\usepackage{booktabs}
\usepackage{amsmath}
\usepackage{amssymb}
\usepackage{graphicx}
\usepackage{xcolor}
\usepackage{amsthm}
\usepackage{caption}
\usepackage{hyperref}
\usepackage{bm}
\usepackage{multirow}
\usepackage{cite}
\usepackage{algorithmic}

\theoremstyle{definition}

\newcommand{\mypar}[1]{{\bf #1.}}
\newtheorem{myLem}{Lemma}

\newtheorem{myThm}{Theorem}

\makeatletter
\newcommand\bigDiamond{\mathop{\mathpalette\bigDi@mond\relax}}
\newcommand\bigDi@mond[2]{%
  \vcenter{\hbox{\m@th
    \scalebox{\ifx#1\displaystyle 2\else1.2\fi}{$#1\Diamond$}%
  }}%
}
\newcommand\bigLozenge{\mathop{\mathpalette\bigL@zenge\relax}}
\newcommand\bigL@zenge[2]{%
  \vcenter{\hbox{\m@th
    \scalebox{\ifx#1\displaystyle 2\else1.2\fi}{$#1\blacklozenge$}%
  }}%
}
\makeatother

\newcommand{\R}{\ensuremath{\mathbb{R}}}

\def\x{\mathbf{x}}

\def\y{\mathbf{y}}

\def\tt{\mathbf{t}}

\ifCLASSINFOpdf
\else
\fi
\begin{document}
%
\title{Online Multi-Agent Forecasting with
Interpretable Collaborative Graph Neural Networks}
%
%
%

\author{Maosen Li,~\IEEEmembership{Student Member,~IEEE,} ~
        Siheng Chen,~\IEEEmembership{Member,~IEEE,} ~
        Yanning Shen,~\IEEEmembership{Member,~IEEE,} ~
        Genjia Liu, ~ 
        Ivor Tsang,~\IEEEmembership{Senior Member,~IEEE,} ~
        and~Ya Zhang,~\IEEEmembership{Member,~IEEE}
\thanks{M. Li and S. Chen are with Cooperative Medianet Innovation Center, Shanghai Jiao Tong University, Shanghai, 200240, China.  E-mail: \{maosen\_li, sihengc\}@sjtu.edu.cn.}
\thanks{Y. Shen is with University of California Irvine, California, The United States. E-mail: yannings@uci.edu. }
\thanks{G. Liu is with Shanghai Jiao Tong University, Shanghai, 200240, China. E-mail: LGJ1zed@sjtu.edu.cn}
\thanks{I. Tsang is with Australian Artificial Intelligence Institute, University of Technology Sydney, Sydney, NSW, Australia. E-mail: ivor.tsang@uts.edu.au}
\thanks{Y. Zhang is with Cooperative Medianet Innovation Center, Shanghai Jiao Tong University, Shanghai, 200240, China.  E-mail: ya\_zhang@sjtu.edu.cn.}
}

\maketitle

\begin{abstract}
This paper considers predicting future statuses of multiple agents in an online fashion by exploiting dynamic interactions in the system. We propose a novel collaborative prediction unit (CoPU), which aggregates the predictions from multiple collaborative predictors according to a collaborative graph. Each collaborative predictor is trained to predict the status of an agent by considering the impact of another agent. The edge weights of the collaborative graph reflect the importance of each predictor. The collaborative graph is adjusted online by multiplicative update, which can be motivated by minimizing an explicit objective. With this objective, we also conduct regret analysis to indicate that, along with training, our CoPU achieves similar performance with the best individual collaborative predictor in hindsight.  This theoretical interpretability distinguishes our method from many other graph networks. To progressively refine predictions, multiple CoPUs are stacked to form a collaborative graph neural network. Extensive experiments are conducted on three tasks: online simulated trajectory prediction, online human motion prediction and online traffic speed prediction, and our methods outperform state-of-the-art works on the three tasks by $28.6\%$, $17.4\%$ and $21.0\%$ on average, respectively. 
\end{abstract}


%
\IEEEpeerreviewmaketitle

\section{Introduction}
%
%
%
%
\IEEEPARstart{D}{ynamic} multi-agent systems depict a group of co-evolving agents which perform interactions to activate or constrain the pattern of each other. Dynamic systems are ubiquitous and critical in many real-world scenarios, such as transportation networks and smart grids~\cite{Jahangiri_2015_TITS_AML,Wojtusiak_2012_CMA_MLA,Yu_2015_ACMSIGAPP_TSM}. In many applications, we need to predict future measurements of a multi-agent system, such as the traffic on various roads within hours for route planning~\cite{Song_2020_AAAI_STS}. 

To forecast the future status of systems, previous works usually train models with finite training datasets and evaluate models with unseen test sets in offline settings. The offline models essentially assume the training and test samples are stationary and share the same distribution, thus fixed models could generalize to arbitrary systemic states~\cite{Kipf_2018_ICML_NRI,Graber_2020_CVPR,NEURIPS2020_e4d8163c}. We sketch the offline learning paradigms in Fig.~\ref{fig:Jewel} (a).
\begin{figure}[tb]
    \centering
    \includegraphics[width=0.49\textwidth]{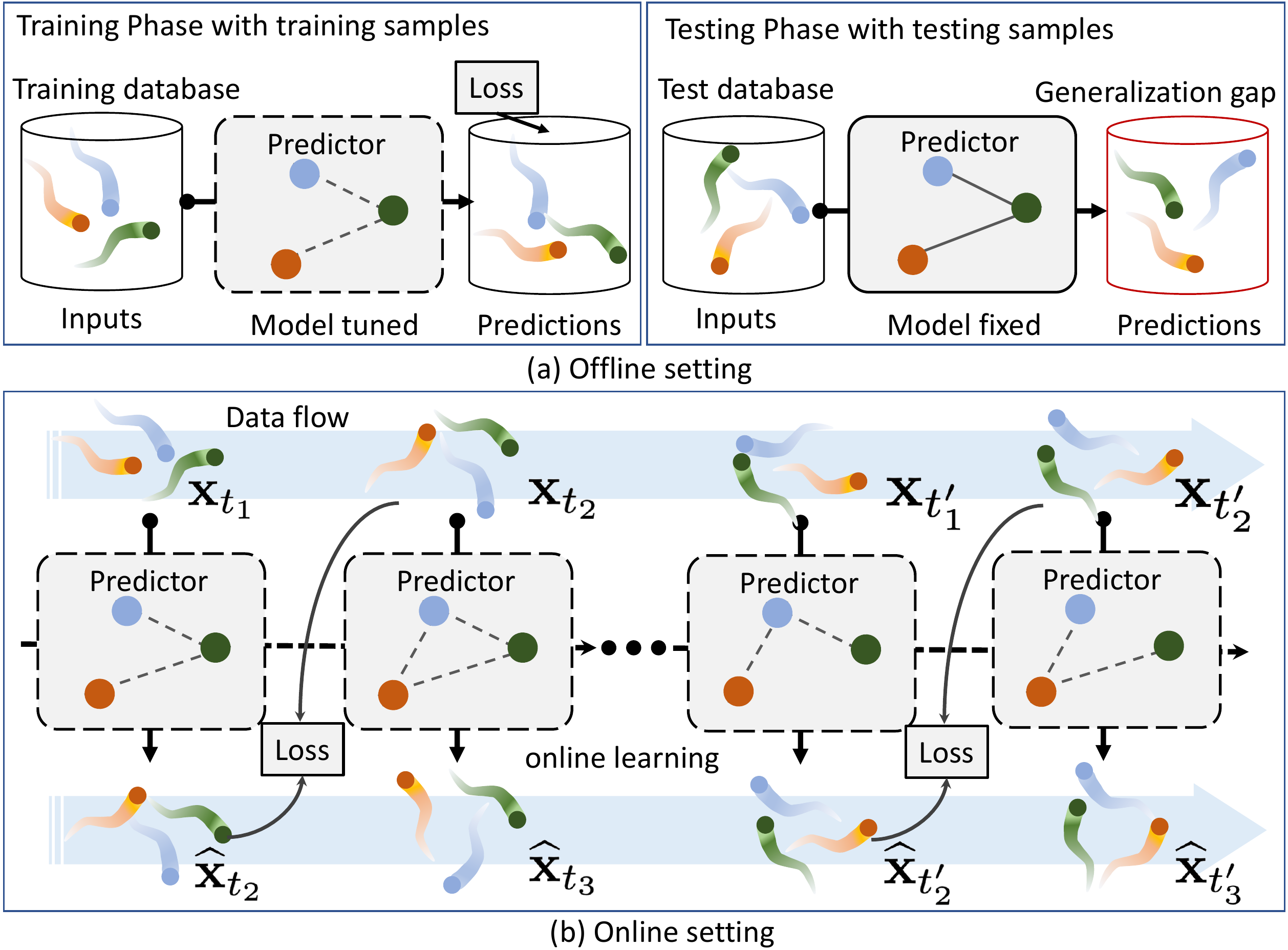}
    \caption{\small Comparison of the graph-based predictors in offline and online settings. (a) For the offline multi-agent forecasting, a training database is used to tune the predictor, followed by an isolated test phase with a test database. In the test phase, the model and graph construction mechanism are fixed; however, the highly dynamic data and possibly outlier test distribution might cause a generalization gap. We use a dashed box with a graph to represent trainable models and a solid box with a graph for fixed model and graph construction. (b) For online learning, the predictor is tested and tuned simultaneously along with the unlimited iterations. The model could dynamically adapt to time-varying data for real-time prediction. Note that predictions are supervised by data flow at delayed iterations.}
    \label{fig:Jewel}
\end{figure}
However, limited training sets are still hard to depict complicated real-world scenarios~\cite{Hoi_2018_arxiv_OLA}, and the multi-agent system could be highly dynamic, resulting in the infinitely possible states changing over time. To make reliable predictions for the dynamic multi-agent system with streaming states, we need an online model to adapt to a potential varying environment in real-time. Taking streaming data captured along time as inputs, the online multi-agent forecasting model is tested and trained simultaneously step-by-step and never stops. The online learning methods never fix the model parameters, while they always tune the models to  capture the dynamic even infinite systemic states along with iteration. The online learning paradigm is sketched in Fig.~\ref{fig:Jewel} (b).

Moreover, to learn the pattern of multiple agents, previous methods usually regard every single agent in an isolated manner, such as autoregressive models~\cite{Zivot2006}, hidden Markov models~\cite{Lehrmann_2014_CVPR_ENM}, Kalman filter~\cite{Vishwakarma_1970_IJSS_POE}, and deep sequence-to-sequence models~\cite{Hochreiter_1997_NeurComput_LSTM}, which exhibit increasingly powerful modeling; however, for multi-agent forecasting, these methods might underestimate potential interactions among the co-evolving agents in a dynamic system. For example, traffic flows on neighboring streets are mutually dependent; and human body-parts might constrain each other during moving. To leverage such relations, some works build graphs to model the relations for representation learning~\cite{Guo_2019_AAAI_ABS}. For example, in traffic analysis,~\cite{Yu_2018_IJCAI_STG} builds a road network according to the urban traffic map. To capture more flexible relations, some works build trainable graphs as model parameters, whose edge weights are updated in an end-to-end setting~\cite{Mao_2019_ICCV_LTD,Cui_2020_CVPR_LDR}. However, such graphs are typically constructed through trials and errors, lacking direct guidance and theoretical interpretation, leading to the difficulties of rigorous understanding and reasonable improvement. In this work, we propose interpretable update steps to learn a graph that captures the collaboration among agents. We further provide the theoretical justification of this graph.

This work aims to design an online forecasting model, which captures the dynamic interactions among agents for real-time prediction. The core component of our method is a \emph{collaborative prediction unit} (CoPU), which automatically builds a trainable \emph{collaborative graph} to represent the systematic interactions. This collaborative graph includes two types of nodes: agents and collaborative pairs. A collaborative pair, essentially formed by a pair of agents, refers to the mutual collaboration between the two agents. Fed with collaborative pairs, a collaborative predictor is built in CoPU to forecast each agent's status, aiming to leverage the inter-agent impacts, thus the collaborative predictors produce the possible measurements for each agent according to the effects from another agent. Finally, we learn a weight on each edge of the collaborative graph, aggregating all collaborative predictors for the final predictions. The trainable edge weights reflect the influence level of each collaborative pair. Notably, in CoPU, the collaborative predictor is parameterized and trained through online gradient descent, while, in a different way, the edge weights are adjusted by an online multiplicative update step-by-step to reflect the dynamic topology. The update strategy can be motivated by solving an explicit objective function. To understand CoPU, we conduct regret analysis and prove that the performance of CoPU can converge to that of the best single collaborative predictor.

One advantage of the proposed method is its interpretability, which comes from three aspects. 1) \emph{Algorithm}: the edge weights of our collaborative graph are obtained by solving an explicit optimization,  while previous works are based on heuristic designs~\cite{Kipf_2018_ICML_NRI, Li_2020_CVPR_DMG, Li_2018_ICLR_Diffusion, Li_2019_CVPR_ASG, Hu_2020_CVPR_CMP, Guo_2019_AAAI_ABS}. 2) \emph{Theory}: we prove that the gap between the predictions based on our collaborative graph and the optimal collaborative pair is within a bounded range. In other words, our collaborative graph describes the information and importance carried by different collaborative pairs for prediction. In comparison, previous works have no theoretical support at all. 3) \emph{Visualization}: our interpretability is reflected based on human understanding. In experiments, we visualize the learned collaborative graphs changing over time, reflecting comprehensible, reasonable and dynamic graphs in the online setting.

To produce more precise predictions, we further stack a series of CoPUs to build a \emph{collaborative graph neural network} (CoGNN), progressively refining the predictions. In CoGNN, each CoPU with a collaborative graph is directly supervised by the ground-truth targets, thus all CoPUs predict agents' statuses in the original measurement space. Skip-connections across CoPUs are leveraged to stablize prediction.
Compared to many spatio-temporal graph neural networks~\cite{Li_2020_CVPR_DMG, Shi_2019_CVPR_TSA, Yu_2018_IJCAI_STG}, the advantages of CoGNN include its interpretability and simplicity, because the information aggregation is performed according to the importance of collaborative pairs in the clear output space, while previous methods generate graph features in the hidden space.

We conduct extensive experiments to validate the proposed CoGNNs on three important tasks: online simulated trajectory prediction, online human motion prediction and online traffic speed prediction. We note that we formulate a new online setup that simultaneously trains and tests models, and we adapt previous offline state-of-the-art methods to the online setting. According to the results, CoGNNs significantly outperform state-of-the-art methods. 

The main contributions are as follows.
\begin{itemize}
    \item We propose a novel collaborative prediction unit (CoPU), which predicts the future statuses of multiple correlated agents based on a collaborative graph. The collaborative graph contains two types of nodes; that is, individual agents and collaborative pairs, which represent the agent dynamics and collaboration status between agents. In this way, an agent could share customized information with each neighboring agent via the collaborative pair.
    \item We employ an online multiplicatively update algorithm to learn the edge weights of the collaborative graph on the fly to solve an explicit objective function; see Eq.~\eqref{eq:w_update}. In this way, the graph tuned and tested simultaneously in an online setting effectively adapts to the changing states of the dynamic data flows. Moreover, the collaborative graph can be theoretically justified through regret analysis.
    \item We propose a novel collaborative graph neural network (CoGNN), which consists of a sequence of CoPUs and progressively refines multi-agent forecasting.
    \item We conduct extensive experiments to validate the effectiveness, interpretation and convergence of our method, which obtains superior and reasonable results.
\end{itemize}


The rest of this paper is organized as follows. In section~\ref{sec:RelatedWorks}, we introduce some works related to multi-agent forecasting, graph deep learning and online learning for prediction. In section~\ref{sec:Formulation}, we define and formulate the problem. In section~\ref{sec:CoPU}, we construct the proposed CoPU, introduce the online training method, and conduct the theoretical regret analysis. In section~\ref{sec:CoGNN}, we construct the entire CoGNN model. Finally, the experiments validating the advantages of our model and the conclusion of the paper are provided in Section~\ref{sec:experiment} and Section~\ref{sec:conclusion}.

\section{Related Works}
\label{sec:RelatedWorks}
In this section, we introduce critical methods related to our works from aspects of 1) multi-agent time-series forecasting, 2) graph deep learning and 3) online learning for prediction.

\mypar{Multi-agent forecasting of time-series}
Multi-agent forecasting has attracted huge attention and brought wide applications.  In the early era, state models were studied~\cite{Williams_2003_JTE_MFV, Lehrmann_2014_CVPR_ENM, Wang_2006_NIPS_GPD, Taylor_2019_ICML_FCR}.  Recently, data-driven models extract deep representations.  DeepState~\cite{Rangapuram_2018_NIPS_DSS}, Res-sup.~\cite{Martinez_2017_CVPR_OHM}, AGED~\cite{Gui_2018_ECCV} and LSTNet~\cite{Lai_2018_SIGIR_MLS} design recurrent networks for stability.  TCN~\cite{Bai_2018_arxiv_AEE} proposes feed-forward convolution for temporal dependencies.  However, the methods do not consider the informative interactions among agents to benefit forecasting.

Recently, some works construct graphs to explicitly exploit the agent relations. NRI~\cite{Kipf_2018_ICML_NRI} infers graphs via an autoencoder for prediction. S-RNN~\cite{Jain_2016_CVPR_SRD} builds factor graphs to propagate agents' information. DCRNN~\cite{Li_2018_ICLR_Diffusion} and ST-GCN~\cite{Yu_2018_IJCAI_STG} incorporate distance-based graphs for traffic forecasting. GraphWaveNet~\cite{Wu_2019_IJCAI_GWD} and Traj-GCN~\cite{Mao_2019_ICCV_LTD} use adaptive topologies to capture potential dependencies. DMGNN~\cite{Li_2020_CVPR_DMG} builds multiscale graphs. However, most graphs are built with limited human priors or designed through trials and errors, lacking theoretical guarantees and interpretation. In this work, we develop a novel framework for multi-agent forecasting, which infers a collaborative graph with theoretical justification.

\mypar{Graph deep learning}
The expressiveness of graphs contributes to various scenarios such as social networks~\cite{Li_2020_NIPS_GGN}, bioinformatics networks~\cite{Dobson_2003_JMB_DES} and human behaviors~\cite{Li_2019_CVPR_ASG,Hu_2020_CVPR_CMP}.  Graph neural networks (GNNs)~\cite{Kipf_2017_ICLR_SSC, Velickovic_2018_ICLR_GAN,9069948,9064693}, which expand deep leading to the non-Euclidean domain, have attracted explosive interests. GNNs can be mainly categorized into the spectral-domain-based~\cite{Bruna_2014_ICLR_SNL,Defferrard_2019_NIPS_CNN,Kipf_2017_ICLR_SSC} and vertex-domain-based~\cite{Hamilton_2017_NIPS_IRL,Niepert_2016_ICML_LCN,Velickovic_2018_ICLR_GAN,Li_2016_ICLR_GGS,Dai_2016_ICML_DEL,Xu_2019_ICLR_HPA,9079208}, which respectively learn the patterns from the graph Fourier representations or the raw topologies.  In this work, we learn interpretable graphs for multi-agent forecasting, which captures dynamic and complex correlations for precise prediction.

\mypar{Online learning for prediction}
A dynamic system evolves with potentially infinite states over time, performing on online streaming data~\cite{Thrun_1998_L2L_LLA}.  Tailored for the data flows, online prediction models have been designed to update in real-time~\cite{Lu_2018_TOIS_SPA}. Traditional kernel-based works develop budgeted kernel learning~\cite{Kivinen_2004_TSP_OLK, Dekel_2008_SIAM_TFA, Wang_2012_JMLR_BCK,9108601}, RF approximations~\cite{Lu_2016_JMLR_LSOK, Bouboulis_2018_TSP_ODL, Ding_2017_ICDM_LSK} and multi-kernel learning methods~\cite{Shen_2019_JMLR_RFO}. Recently, some deep-learning-based methods are proposed for online prediction~\cite{Xiu_2018_BMVC_PFE, Butepage_2018_ICRA_AMF, Marchetti_2020_CVPR_MMA}, whose network structures are specifically designed.  In this work, we train graph-based multi-agent forecasting models in an online setting, and optimize graphs to depict the agents' interactions in real-time.

\section{Problem Formulation}
\label{sec:Formulation}
With the time-series measurements collected by multiple agents in \emph{real-time}, the task of
online multi-agent forecasting essentially aims to predict the future statuses of each agent at each online time stamp based on the corresponding historical states. 
Mathematically, for the $p$th agent $v^{(p)}$, at time stamp $t$, let $\x_{t-\delta:t}^{(p)}\in\mathbb{R}^{\delta \times d}$ be the measurements of the observed clip with length $\delta$ and feature dimension $d$, recording $v^{(p)}$'s states within time interval $(t- \delta,t]$; and $\x_{t:t+\delta}^{(p)}\in\mathbb{R}^{\delta \times d}$ be the measurements of the corresponding ground-truth clip in the future, reflecting $v^{(q)}$'s states within $(t,t+\delta]$. We note that the length of historical and future sequences could be different, while we here use $\delta$ to simplify notation. Our~\emph{online} predictor $f_t(\cdot)$ aims to produce 
$$\{ \widehat{\x}_{t:t+\delta}^{(p)} \}_{p=1}^N = f_t ( \{ \x_{t-\delta:t}^{(p)} \}_{p=1}^N )
$$ 
to approximate the ground-truth $\{ \x_{t:t+\delta}^{(p)} \}_{p=1}^N$, where the input data $\x_{t-\delta:t}^{(p)}$ is captured from streaming data flow at time stamp $t$. Note that the model $f_t(\cdot)$ indexed by $t$ is updated online. The online setting is crucial because the online predictor can effectively adapt to the unseen, or highly dynamic data distribution.

In contrast, previous works~\cite{Kipf_2018_ICML_NRI,Hu_2020_CVPR_CMP,Mao_2019_ICCV_LTD,Cui_2020_CVPR_LDR,Cai_2020_ECCV_LPJ} mostly consider an \emph{offline} setting, which includes training and testing phases. In the training phase, a predictor is optimized based on a training dataset; in the testing phase, the predictor is fixed and deployed on a testing dataset.  Mathematically, let $f(\cdot)$ be an offline predictor, $\mathcal{X}_{\rm train}$ be the training dataset and $\mathcal{X}_{\rm test}$ be the testing dataset, the training and testing phases work as 
\begin{equation*}
    \begin{aligned}
          \mathrm{Training}: \{ \widehat{\x}_{0:\delta}^{(p)} \}_{p=1}^N &= f ( \{ \x_{-\delta:0}^{(p)} \}_{p=1}^N ), ~~ {\bf x} \sim \mathcal{X}_{\rm train}, \\
          \mathrm{Testing}: \{ \widehat{\y}_{0:\delta}^{(p)} \}_{p=1}^N &= f ( \{ \y_{-\delta:0}^{(p)} \}_{p=1}^N ), ~~ {\bf \y} \sim \mathcal{X}_{\rm test}.
    \end{aligned}
\end{equation*}
These offline methods assume the training and testing time-series should share the same data distribution and the resulting offline predictors would be hardly adapt to domain shift in the testing phase.

To design an online predictor, our core strategy is to exploit the dynamic mutual relations of agents and capture their evolution status through online learning. To model pairwise collaborations, we introduce a key concept, \emph{collaborative pair}, which groups two agents and can be used to reflect the effect from one agent to another in a multi-agent system. For example, the $p$th agent $v^{(p)}$ and the $q$th agent $v^{(q)}$ form a collaborative pair $(v^{(p)},v^{(q)})$ to indicate the directed effect from $v^{(q)}$ to $v^{(p)}$. In this way, to predict the future status of the $p$th agent, we consider $\widehat{\x}_{t:t+\delta}^{(p)} = f_{t} ( \{\x_{t-\delta:t}^{(p)}, \x_{t-\delta:t}^{(q)}\}_{q=1}^{N} )$, which exploits information from all the collaborative pairs of $v^{(p)}$ to benefit its prediction. This model considers each collaborative pair as a basic element, which is different from previous works that regard each individual agent as a basic element~\cite{Mao_2019_ICCV_LTD,Cui_2020_CVPR_LDR}.

\section{Collaborative Prediction Unit}
\label{sec:CoPU}
To develop an online multi-agent forecasting model, we propose a collaborative prediction unit (CoPU) as a basic module, which predicts the status of each agent from associated collaborative pairs via an collaborative graph.

\subsection{Collaborative graph}
To exploit the interactions among agents, we propose a collaborative graph that explicitly models the pairwise correlations. Let $G(\mathcal{V}, \mathcal{F}, \mathbf{W}_t)$ be a collaborative graph, where $\mathcal{V} = \{v^{(1)}, v^{(2)}, \cdots, v^{(N)} \}$ is the agent set with $v^{(p)}$ modeling the $p$th agent, $\mathcal{F} = \{(v^{(p)}, v^{(q)})\}_{p,q=1}^N$ is the set of collaborative pairs with $(v^{(p)}, v^{(q)})$ modeling the $(p,q)$th collaborative pair; and ${\bf W}_t \in \R^{N \times N}$ is the collaborative graph adjacency matrix trained online, whose $(p,q)$th element $w_t^{(p,q)}$ reflects the influence level from $v^{(q)}$ to $v^{(p)}$ at time stamp $t$. Note that this collaborative graph is asymmetric and dynamic; that is, each collaborative edge represents a directional relationship between agents and each collaborative edge weight is time-varying during the online training process.
\begin{figure}[tb]
    \centering
    \includegraphics[width=0.49\textwidth]{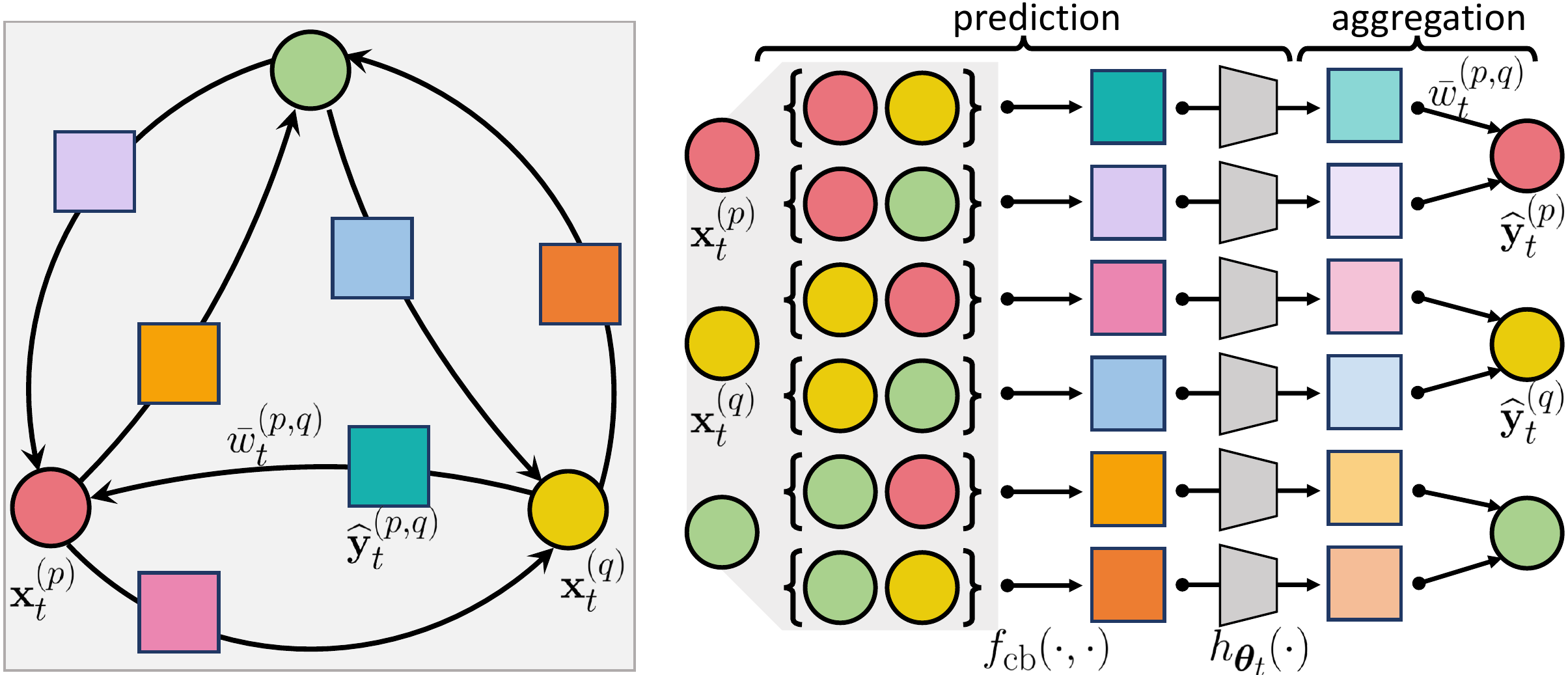}
    \caption{\small Collaborative prediction unit (CoPU) sketched with three agents for example. The left plot shows the CoPU based on a collaborative graph to learn and aggregate information from collaborative pairs, where the circled nodes represents the agents and the squared nodes represents the collaborative pairs. The right plot shows the detailed computation, where the collaborative pairs enables the late aggregation of agents' predictions.}
    \label{fig:copu}
\end{figure}

Different from many ordinary graphs, the collaborative graph has two types of nodes: agent and collaborative pair. Agents are actual nodes; while collaborative pairs are virtual nodes as they are naturally obtained when agents are given, which have directions. To predict the future status of each agent, we rely on each of its associated collaborative pairs to produce a possible future status for such an agent; and then, the final status is obtained by averaging all the possible statuses based on the weighted collaborative edges.

Different from many previous graph-convolution-based and message-passing-based models~\cite{Mao_2019_ICCV_LTD,Cui_2020_CVPR_LDR,Cai_2020_ECCV_LPJ,Kipf_2018_ICML_NRI, Graber_2020_CVPR, NEURIPS2020_e4d8163c}, where each node shares the same node feature to all of its neighbors without considering the demand of each one of its neighbors, our model, based on a collaborative pair, $v^{(q)}$ helps the prediction of $v^{(p)}$ by sharing customized information according to~$v^{(p)}$'s status. 

\subsection{Model design}
\label{sec:model_design}
Based on the collaborative graph, to construct the CoPU, we consider two key operations in series: collaborative-pair-based prediction, which predicts an agent's future status based on an associated collaborative pair, which generate all possible predictions with the collaborative effects between agents, and weighted collaborative-graph-based aggregation, which obtains the final prediction. Fig.~\ref{fig:copu} illustrates the CoPU, where the left plot shows the modeling thoughts, and the right plot shows the detailed computation.

\mypar{Collaborative-pair-based prediction}
To explicitly form a collaborative pair, which represent the collaborative states of two agent in history, we propose a combination function $f_{\rm cb}(\cdot,\cdot)$ to directedly combine the measurements of two agents. Given the observed measurements $\x_{t-\delta:t}^{(p)}$, $\x_{t-\delta:t}^{(q)}\in\mathbb{R}^{\delta \times d}$, the state of the collaborative pair $(v^{(p)}, v^{(q)})$ is
\begin{equation}
\label{eq:collab_pair_formulation}
    \x_{t-\delta:t}^{(p, q)} = f_{\rm cb}(\x_{t-\delta:t}^{(p)}, \x_{t-\delta:t}^{(q)}) \in \mathbb{R}^{\delta \times d'}.
\end{equation}
The collaborative state $\x_{t-\delta:t}^{(p, q)}$ carries the directed effects from the agent $v^{(q)}$ to the agent $v^{(p)}$.

Based on the collaborative pairs and their states, we propose a \emph{collaborative predictor}, which embeds any collaborative pair to predict an agent's future status under the corresponding effects. 
Let the collaborative predictor be $h_{\bm{\theta}_{t}}(\cdot)$, where  $\bm{\theta}_{t}$ is the parameter trained online at time $t$. To predict $v^{(p)}$ under the collaborative effects from $v^{(q)}$, the predicted status is
\begin{equation}
\label{eq:temp_pattern_learner}
    \widehat{\x}^{(p,q)}_{t:t+\delta} = h_{\bm{\theta}_{t}}({\x}_{t-\delta:t}^{(p,q)})\in\mathbb{R}^{\delta \times d}.
\end{equation}
The physical meaning of $\widehat{\x}^{(p,q)}_{t:t+\delta}$ is one possible future trajectory for $v^{(p)}$ by considering the influence from $v^{(q)}$. 

For a collaborative pair, the aggregation function $f_{\rm cb}(\cdot,\cdot)$ and the predictor $h_{\bm{\theta}_{t}}(\cdot)$ could be designed in numerous forms. Here we specifically consider three typical forms:
\begin{itemize}
    \item \emph{Autoregressive model.}
          Let $f_{\rm cb}(\cdot,\cdot)$ concatenate the multi-order differences of two agents, i.e., ${\x}_{t-\delta:t}^{(p, q)} = [\x_{t-\delta:t}^{(p)}, ~\x_{t-\delta:t}^{(q)}-\x_{t-\delta:t}^{(p)}, ~\dots, ~(\x_{t-\delta:t}^{(q)}-\x_{t-\delta:t}^{(p)})^{D}]\in\mathbb{R}^{\delta \times (D+1)d}$, which provides a polynomial distance feature.
          We then vectorize ${\x}_{t-\delta:t}^{(p,q)}$ by reshaping it along time, and directly set a linear $h_{\bm{\theta}_{t}}(\cdot)$ to learn $\widehat{\x}^{(p,q)}_{t:t+\delta}$.
    \item \emph{LSTM model.}
          Let $f_{\rm cb}(\cdot,\cdot)$ concatenate the measurements of the two agents, i.e., ${\x}_{t-\delta:t}^{(p,q)}=[\x_{t-\delta:t}^{(p)}, \x_{t-\delta:t}^{(q)}]\in\mathbb{R}^{\delta \times 2d}$. The concatenation of the raw measurements reduces the dimensions and encourages more flexible information. We employ a one-layer LSTM as $h_{\bm{\theta}_{t}}(\cdot)$ on ${\x}_{t-\delta:t}^{(p,q)}$ to extract the temporal dynamics.
    \item \emph{Temporal convolution model.}
          We also consider $f_{\rm cb}(\cdot,\cdot)$ as the raw measurement concatenation.
          We then perform two layers of 1-D convolution along time to capture sequential dynamics. A ${\tanh(\cdot)}$ function is used as the activation after the first convolution to enhance flexibility.
\end{itemize}
Each agent has $N$ collaborative pairs, so we predict $N$ possible future statuses for one agent.  To obtain the final prediction, we need to adaptively aggregate the $N$ possibilities. 

\mypar{Collaborative-graph-based aggregation}
Given the possible future status for an agent, which is predicted based on each associated collaborative pair, we aggregate these statuses with different weights on the edge of the collaborative graph, where the edge weights are updated online and reflect the influence level from one agent to another.
Mathematically, let the graph adjacency matrix at iteration step $t$ be ${\bf W}_t$, whose the $(p,q)$th element is ${w}_{t}^{(p,q)}$.
The final prediction of $v^{(p)}$ under the effects from all the associated collaborative pairs is formulated as
\begin{equation}
    \widehat{\x}_{t:t+\delta}^{(p)} = 
    \frac{\sum_{q=1}^{N} {w}_{t}^{(p,q)} \widehat{\x}_{t:t+\delta}^{(p,q)}} {\sum_{q'=1}^{N} {w}_{t}^{(p,q')} },
\end{equation}
where the edge weights are normalized for any agent $v^{(p)}$.
In all, the feed-forward operations of the CoPU is formulated as
\begin{equation}
\label{eq:entire_copu}
    \widehat{\x}_{t:t+\delta}^{(p)} = 
    \sum_{q=1}^{N}
    \bar{w}_{t}^{(p,q)}
    h_{\bm{\theta}_{t}} 
    (
        f_{\rm cb}(\x_{t-\delta:t}^{(p)}, \x_{t-\delta:t}^{(q)})
    ),
\end{equation}
where $\bar{w}_{t}^{(p,q)} = {w}_{t}^{(p,q)}/{\sum_{q'=1}^{N} {w}_{t}^{(p,q')} }$ denotes the normalized collaborative edge weights.

Compared to most common graph-based models, which bridge the same types of nodes with a scalar edge weight, the proposed CoPU considers the heterogeneous nodes of agents and collaborative pairs, carrying individual dynamics the interactive effects in a hybrid manner. The CoPU essentially aggregates the comprehensive information from the collaborative pairs with different importance.
Moreover, the proposed CoPU has a similar framework with the online expert mixture algorithm~\cite{Hazan_2016_FTM_IOC,Shen_2019_JMLR_RFO}, as the different collaborative predictors associated with an agent could be regarded as experts; however, each collaborative predictor explicitly exploits the customized information for the associated agent based on the corresponding collaborative pair, instead of producing an untargeted prediction independently like most expert mixture methods.

\subsection{Optimization}
CoPU enables a complete multi-agent forecasting pipeline, and the internal parameters could be trained given the prediction targets.
To train the CoPU online, at each iteration $t$, we update the model parameters to minimize the loss function:
\begin{equation}
    \mathcal{L} = 
    \sum_{t=1}^T 
    \sum_{p=1}^N 
    \mathcal{L}_t 
    \left( 
        \sum_{q=1}^N 
        \bar{w}_{t}^{(p,q)}
        h_{\bm{\theta}_{t}}
        \left(
            \x_{t-\delta:t}^{(p,q)}
        \right), \x_{t:t+\delta}^{(p)} 
    \right),
    \label{eq:original_loss}
\end{equation}
where $\mathcal{L}_t(\cdot)$ denotes a fixed and convex loss function evaluated at time $t$; $T$ is the length of iterations before the current stamp. 
Since $\mathcal{L}_t(\cdot)$ is a convex function (e.g., $\ell_2$ loss) w.r.t. $h_{\bm{\theta}_{t}}(\x_t^{(p,q)})$, we consider an upper bound of it,
\begin{equation}
    \mathcal{L} :=
    \sum_{t=1}^T  
    \sum_{p,q=1}^N 
        \bar{w}_{t}^{(p,q)}  
        \mathcal{L}_t 
        \left( 
            h_{\bm{\theta}_{t}} 
            \left(
                \x_{t-\delta:t}^{(p,q)} 
            \right),  
            \x_{t:t+\delta}^{(p)} 
        \right).
    \label{eq:practical_loss}
\end{equation}
To minimize Eq.~\eqref{eq:practical_loss}, 
instead of using common online gradient descent to update $\bar{w}_{t}^{(p,q)}$ and ${\bm{\theta}_{t}}$ together, we consider different updating mechanisms, which bring theoretical guarantees.
We separately optimize the parameters $\bm{\theta}_{t}$ and $\bar{w}_{t}^{(p,q)}$ with different steps. First, at time $t$, the parameter $\bm{\theta}_{t}$ could be trained by online gradient decent; that is,
\begin{equation*}
    \bm{\theta}_{t+1} ~ \leftarrow ~ \bm{\theta}_{t} -\eta \:  \nabla{\cal L}_t \left(h_{\bm{\theta}_{t}}^{(p,q)} \right),
\end{equation*}
where we use ${\cal L}_t(h_{\bm{\theta}_{t}}^{(p,q)})$ to simplify ${\cal L}_t(h_{\bm{\theta}_{t}} (\x_{t-\delta:t}^{(p,q)}), \x_{t:t+\delta}^{(p)})$; $\eta\in(0,1)$ is the learning rate. Next, for $\bar{w}_{t}^{(p,q)}$, we employ a multiplicative exponentiated update, which is inspired from the randomized weighted majority of online learning~\cite{Hazan_2016_FTM_IOC}. We assume the optimal edge weight to be $w^{(p,q)}\in[0,1]$, and $\bar{w}_{t}^{(p,q)}$ is updated through
\begin{equation}
\label{eq:w_update}
\begin{aligned}
w_{t+1}^{(p,q)} \leftarrow & \arg\min_{w^{(p,q)}}\, \eta \: 
{\cal L}_t\left( h_{\bm{\theta}_{t} } \right)
\left ({w}^{(p,q)} - \bar{w}^{(p,q)}_{t} \right) \\
& + {\cal D}_{\rm KL}({w}^{(p,q)} \| \bar{w}^{(p,q)}_{t})
\\
= & ~
\bar{w}_{t}^{(p,q)}\exp\left(-\eta \: {\cal L}_t
\left(
h_{\bm{\theta}_{t}}
\right) \right),
\\
\end{aligned}
\end{equation}
where  ${\cal D}_{\rm KL}({w}^{(p,q)}\|\bar{w}_{t}^{(p,q)})$ denotes the KL-divergence between the optimized ${w}^{(p,q)}$ and the collaborative weight at the last iteration step. The optimization ensures the continuity and smoothness of the collaborative graph structure by minimizing the element-wise distances and distribution divergence. We also normalize the new collaborative edge weights by $\bar{w}_{t+1}^{(p,q)} = w_{t+1}^{(p,q)}/\sum_{p=1}^N w_{t+1}^{(p,q)}$. 

In previous graph neural networks~\cite{Li_2020_CVPR_DMG, Shi_2019_CVPR_TSA, Mao_2019_ICCV_LTD, Cui_2020_CVPR_LDR}, graphs are obtained through either heuristic designs or end-to-end learning in a black-box, which lack a clear objective and theoretical justification. Compared to these works, the proposed method obtains the edge weights of the collaborative graph through optimizing an explicit objective function~\eqref{eq:w_update}, which could be derived as an update function based on the input features and historical states in an online setting. 

\subsection{Comparison with previous works}

The proposed CoPU is seemingly similar to spatio-temporal-graph convolution operations~\cite{Li_2020_CVPR_DMG, Shi_2019_CVPR_TSA, Yu_2018_IJCAI_STG} and neural-message-passing operations~\cite{Kipf_2018_ICML_NRI, Graber_2020_CVPR, NEURIPS2020_e4d8163c}. Here we clarify the differences;

Compared to spatio-temporal-graph convolution operations~\cite{Li_2020_CVPR_DMG, Shi_2019_CVPR_TSA, Yu_2018_IJCAI_STG}, the proposed CoPU is novel from three aspects:
\begin{itemize}
        \item Interpretable graph learning. Previous models usually build non-grid structures by using predefined graphs or adjusting the graph in an intractable black-box. Each CoPU contains an interpretable graph that optimizes an explicit objective function for model training and has a theoretical implication.
        \item Customized feature propagation. Previous works consider each node shares the same feature to all of its neighbors without considering the demands of neighbors. In CoPU, a node shares a customized feature to each neighbor via the collaborative pairs.
        \item Interpretable output. Previous methods build deep graph networks to learn the agent's dynamics in the hidden high-dimensional spaces, lacking the interpretation of feature representation; while each CoPU predicts agents' future statuses, which lie in the original measurement space.
\end{itemize}

Previous neural-message-passing-based works~\cite{Kipf_2018_ICML_NRI, Graber_2020_CVPR, NEURIPS2020_e4d8163c} include two steps to propagate information: v2e, which transforms the information of a pair of nodes to an edge, and e2v, which aggregates the information of edges to a node. Compared to those methods, the proposed CoPU is novel from three aspects:
    \begin{itemize}
        \item  Interpretable graph learning. Previous  works either build predefined graphs based on physical constraints or learn graph structures with errors and trials in black-boxes; while the edge weights of our collaborative graph are obtained by optimizing an explicit objective function and has theoretical justification.

        \item Meaning of v2e. In the v2e step, previous works update the embedding of each edge, which is an intermediate-feature-level fusion; while our CoPU directly generates the prediction of a node conditioned on another node, which is an output-level fusion.
 
        \item Meaning of e2v. In the e2v step, previous works consider the mean aggregation, which assigns the same weight for each one of the neighbors; while our CoPU adopts the weighted aggregation, where the weights are the edge weights in the collaborative graph, obtained through multiplicative updating.
    \end{itemize}

\subsection{Theoretical interpretation}
To understand CoPU, we conduct regret analysis. We consider the best single collaborative predictor as our regret baseline and analyze the gap from our CoPU to this predictor. The best predictor has a straightforward meaning and its corresponding collaborative pair makes the most significant effect in predicting an agent's future status. We aim to demonstrate that the regret tends to be finite in the online setting, which indicates the properties of our CoPU and the training algorithm from two aspects:
\begin{itemize}
    \item The CoPU trained online could converge to a statistically stable state during the training process. 
    \item The performance of CoPU achieves a limited gap from the best single collaborative pair in statistics, reflecting that CoPU captures significant and reasonable collaborations for effective and interpretable prediction.
\end{itemize}

Here we define the collaborative pair in the best collaborative predictor.  For the $p$th agent, it has $N$ collaborative pairs,  $\{(v^{(p)},v^{(q)})\}_{q=1}^N$.
The collaborative pair $(v^{(p)}, v^{(q^*)})$ for $v^{(p)}$ in the best collaborative predictor is obtained via
\begin{equation}
    q^* = \arg\min_q  \sum_{t=1}^T \mathcal{L}_t \left( h_{\bm{\theta}^*} \left( {\bf \x}_{t-\delta:t}^{(p,q)}\right)  \right),
\end{equation}
where $\bm{\theta}^*$ denotes the parameters of the best collaborative predictor at iteration step $T$.
However, it is hard to implement an online model by directly determining $(v^{(p)}, v^{(q^*)})$ and meanwhile optimizing $\bm{\theta}^*$; that is why CoPU builds a collaborative graph with trainable weights to achieve a soft collaborative mechanism; see Eq.\eqref{eq:entire_copu}.

To measure the performance gap between the best collaborative predictor and CoPU, we define the static regret based on the difference between the accumulated loss of CoPU and that of the best collaborative predictor in hindsight,
\begin{equation*}
{\rm R}_{T} =
\frac{1}{T}
\sum_{t=1}^T
  \left (
      \mathcal{L}_t
        \left (
          \sum_{q=1}^N
            \bar{w}_{t}^{(p,q)}h_{\bm{\theta}_{t}}^{(p,q)}
        \right )
    -
      \mathcal{L}_t
        \left (
          h_{\bm{\theta}^*}^{(p,q^*)}
        \right )
\right ),
\end{equation*}
where we simplify $h_{\bm{\theta}_t}(\x_t^{(p,q)})$ as $h_{\bm{\theta}_t}^{(p,q)}$.
The regret reflects the gap between the CoPU and the best collaborative predictor, as well as the convergence rate of online learning.

To study the CoPU and the performance gap, we aim to calculate an upper bound of the regret. To facilitate the closed-form theoretical analysis, we consider a basic predictor (a linear autoregressive model) and propose the following assumptions that are commonly satisfied by our model:

\noindent\textbf{(AS1)}
\emph{For the collaborative predictor at time $t$, $h_{\bm{\theta}_{t}}$, given any collaborative pair, the loss ${\cal L}_t(h_{\bm{\theta}_{t}})$  is convex w.r.t. ${\bm{\theta}_t}$.}

\noindent\textbf{(AS2)}
\emph{For any bounded $\bm{\theta}_{t}$, where 
$\|\bm{\theta}_{t}\|_2\leq C_{\bm{\theta}}$, and any collaborative pair, we have a bounded loss, 
${\cal L}_t( h_{\bm{\theta}_{t}}) \in [0,1]$, and a bounded gradient, 
$ \| \nabla {\cal L}_t ( h_{\bm{\theta}_{t}}) \|_2 \leq L $.}

Then, we propose two lemmas, which help to derive the upper bounds of regret. We first analyze the loss gap between arbitrary collaborative predictors and the best collaborative predictor.
\begin{myLem}
\label{Lem1}
Let $h_{\bm{\theta}^*}^{(p,q^*)}$ and $h_{\bm{\theta}_{t}}^{(p,k)}$ be the predictions of the best collaborative predictor and arbitrary one. Under {\bf AS1} and {\bf AS2}, for any collaborative pair $(v^{(p)}, v^{(k)})$, we have
\begin{small}
\begin{equation*}
    \frac{1}{T}
    \sum_{t=1}^{T}
    \left (
    \mathcal{L}_t 
      \left (
        h_{\bm{\theta}_{t}}^{(p,k)}
      \right )
    -
    \mathcal{L}_t 
      \left (
        h_{\bm{\theta}^*}^{(p,q^*)}
      \right )
    \right )
    \leq
    \frac{C_{\bm{\theta}}^2}{2\eta T}
    +
    \frac{\eta L^2}{2},
\end{equation*}
\end{small}
where $\eta$ is the learning rate; $C_{\bm{\theta}}$ is the upper bound of the norm of parameters; and $L$ is the Lipschitz constant.
\end{myLem}

\begin{proof}
Given any determined parameter $\bm{\theta}$ of the collaborative predictor, for any collaborative pair $(v^{(p)},v^{(k)})$, we have
\begin{small}
\begin{equation}
\label{Eq:Apx1}
        \| \bm{\theta}_{t+1} - \bm{\theta} \|_2^2 
        = \| \bm{\theta}_{t} - 
             \eta \nabla^{\top} \mathcal{L}_t(h_{\bm{\theta}_{t}}^{(p,k)}) -
             \bm{\theta} \|_2^2 
\end{equation}
\end{small}
Considering the convexity property of the loss function (see {\bf AS1}), the gradient of the loss function satisfies
\begin{small}
\begin{equation}
\label{Eq:Apx2}
    \nabla^{\top}
    \mathcal{L}_t(h_{\bm{\theta}_{t}}^{(p,k)})
    \geq
    \frac{\mathcal{L}_t(h_{\bm{\theta}_{t}}^{(p,k)})
         -\mathcal{L}_t(h_{\bm{\theta}}^{(p,k)})}
         {\bm{\theta}_{t} - \bm{\theta}}
\end{equation}
\end{small}
Plugging Inequality~\eqref{Eq:Apx2} into Eq.~\eqref{Eq:Apx1}, we could easily obtain the following inequality by rearranging some terms,
{\begin{small}
\begin{equation}
\label{eq:Apx3}
    \begin{aligned}
    & \mathcal{L}_t(h_{\bm{\theta}_{t}}^{(p,k)})
         -\mathcal{L}_t(h_{\bm{\theta}}^{(p,k)}) \\
         \leq ~ &
         \frac{\eta}{2}
         \| \nabla 
         \mathcal{L}_t(h_{\bm{\theta}_{t}}^{(p,k)}) \|_2^2 
         + \frac{\| \bm{\theta}_{t} - \bm{\theta} \|_2^2 -
                 \| \bm{\theta}_{t+1} - \bm{\theta} \|_2^2}
                {2 \eta}
    \end{aligned}
\end{equation}
\end{small}
}
We further consider the assumption {\bf AS2} that points out the upper bound of $\|\bm{\theta}_{t}\|_2$ and the Lipschitz constant of the gradient $\nabla \mathcal{L}_t(h_{\bm{\theta}_{t}})$; that is, for any collaborative predictor, we have $\|\bm{\theta}_{t}\|_2 \leq C_{\bm{\theta}}$ and $\|\nabla \mathcal{L}_t(h_{\bm{\theta}_{t}})\|_2 \leq L$. 
We then summing Inequality~\eqref{eq:Apx3} over $t=1,\dots,T$, we could easily derive 
\begin{small}
\begin{equation}
    \label{eq:Apx5}
    \sum_{t=1}^{T}
      \left (
        \mathcal{L}_t(h_{\bm{\theta}_{t}}^{(p,k)})
         -\mathcal{L}_t(h_{\bm{\theta}}^{(p,k)})
      \right )
     \leq
     \frac{C_{\bm{\theta}}^2}{2\eta}
      + \frac{\eta L^2 T}{2}
\end{equation}
\end{small}
where we assume the initial parameter $\bm{\theta}_{1}$ to have a very small value, leading to $\| \bm{\theta}_{1} - \bm{\theta} \|_2^2 \approx \|\bm{\theta}\|_2^2=C_{\bm{\theta}}^2$, and consider $\| \bm{\theta}_{T+1} - \bm{\theta} \|_2^2$ to be non-negative.

Additionally, we divide by the number of iterations $T$ at the both side of Inequality~\eqref{eq:Apx5} to complete the proof.
\begin{small}
\begin{equation}
    \label{eq:Apx6}
    \frac{1}{T}
    \sum_{t=1}^{T}
      \left (
        \mathcal{L}_t(h_{\bm{\theta}_{t}}^{(p,k)})
         -\mathcal{L}_t(h_{\bm{\theta}^*}^{(p,q^*)})
      \right )
    \leq \frac{C_{\bm{\theta}}^2}{2\eta T} + \frac{\eta L^2}{2}
\end{equation}
\end{small}
\end{proof}

Note that Lemma~\ref{Lem1} are hold for any collaborative pair, so that, in hindsight, we could choose $\bm{\theta}$ to be the parameter at iteration step $T$ and use the predicted status $h_{\bm{\theta}^*}^{(p,q^*)}$, which is output from the collaborative predictor corresponding to the single optimal collaborative pair $(v^{(p)},v^{(q^*)})$.

We next analyze the gap between CoPU and an arbitrary collaborative predictor.
\begin{myLem}
\label{Lem2}
Under {\bf AS1} and {\bf AS2}, for any collaborative pair $(v^{(p)},v^{(k)})$ fed into the predictor $h_{\bm{\theta}_{t}}(\cdot)$, it holds that
\begin{small}
\begin{equation*}
        \frac{1}{T}
        \sum_{t=1}^{T}
        \left (
        \sum_{q=1}^{N}
            \bar{w}_{t}^{(p,q)}
            \mathcal{L}_t
              \left (
                h_{\bm{\theta}_{t}}^{(p,q)}
              \right )
      -
      \mathcal{L}_t
              \left (
                h_{\bm{\theta}_{t}}^{(p,k)}
              \right )
        \right )
      \leq 
      \eta + \frac{\log N}{\eta T},
\end{equation*}
\end{small}
where $N$ is the number of agents.
\end{myLem}

\begin{proof}
We first let ${W}_{t}^{(p)}$ be the sum of edge weights associated with agent $v^{(p)}$; that is ${W}_{t}^{(p)}=\sum_{q=1}^{N}w_{t}^{(p,q)}$.
Therefore, ${W}_{t}^{(p)}$ adjusted through exponentiated update at the time $t+1$ could be formulated as
{\begin{small}
\begin{equation}
\label{eq:Apx7}
    \begin{aligned}
        W_{t+1}^{(p)} = 
        \sum_{q=1}^{N} \bar{w}_{t}^{(p,q)}
        \exp 
        \left( 
          - \eta 
          \mathcal{L}_t \left( h_{\bm{\theta}_{t}}^{(p,q)} \right)
        \right)
    \end{aligned}
\end{equation}
\end{small}}
We represent $\mathcal{L}_t \left( h_{\bm{\theta}_{t}}^{(p,q)} \right)$ as $\mathcal{L}_t$ for simplification.
We consider the inequality $\exp (-\eta x) \leq 1 - \eta x + \eta^2 x^2$ for $\forall |\eta| \leq 1$, and Inequality~\eqref{eq:Apx7} satisfies
{\begin{small}
\begin{equation*}
    W_{t+1}^{(p)}
    \leq 
    1 - 
    \eta \sum_{q=1}^{N} 
         \bar{w}_{t}^{(p,q)} 
         \mathcal{L}_{t} 
    +
    \eta^2 \sum_{q=1}^{N} 
         \bar{w}_{t}^{(p,q)}
         \mathcal{L}_{t}^2
\end{equation*}
\end{small}}
Furthermore, according to $1 + x \leq e^x$, we have
{\begin{small}
\begin{equation}
\label{eq:Apx8}
     W_{t+1}^{(p)} \leq \exp
     \Big ( 
         - \eta 
         \sum_{q=1}^{N} 
         \bar{w}_{t}^{(p,q)}
         \mathcal{L}_{t} 
          +
         \eta^2 \sum_{q=1}^{N} 
         \bar{w}_{t}^{(p,q)} 
         \mathcal{L}_{t}^2
     \Big ) 
\end{equation}
\end{small}}
And we then consider the updated situation of Inequality~\eqref{eq:Apx8} along time when $t=1,\dots,T$, we then obtain
{\begin{small}
\begin{equation}
\label{eq:Apx9}
     W_{T+1}^{(p)} \leq \exp
     \Big ( 
         \sum_{t=1}^{T}
         \eta^2 
         \sum_{q=1}^{N} 
         \bar{w}_{t}^{(p,q)}
         \mathcal{L}_{t}^2
         -
         \eta 
         \sum_{q=1}^{N} 
         \bar{w}_{t}^{(p,q)} 
         \mathcal{L}_{t}
     \Big ) 
\end{equation}
\end{small}}
On the other hand, for the agent $v^{(p)}$ and any agent $v^{(k)}$, we have the inequality,
{
\begin{small}
\begin{equation}
\label{eq:Apx10}
    \begin{aligned}
        W_{T+1}^{(p)} 
        \geq~ & w_{T+1}^{(p,k)} \\
        =~ & w_{1}^{(p,k)}
            \exp \left (
                 -\eta \sum_{t=1}^{T}
                       \mathcal{L}_t
                       \left (
                       h_{\bm{\theta}_{t}}^{(p,k)}
                       \right )
                 \right ) \\
    \end{aligned}
\end{equation}
\end{small}}
Combining Inequality~\eqref{eq:Apx9} and Inequality~\eqref{eq:Apx10}, we arrive at
\begin{small}
\begin{equation}
\label{eq:Apx11}
    \begin{aligned}
    & \exp
    \Big ( 
         - \eta 
         \sum_{t=1}^{T}
         \sum_{q=1}^{N} 
         \bar{w}_{t}^{(p,q)} 
         \mathcal{L}_{t} 
         +
         \eta^2 
         \sum_{t=1}^{T}
         \sum_{q=1}^{N} 
         \bar{w}_{t}^{(p,q)}
         \mathcal{L}_{t}^2
    \Big ) \\
    \geq~ &
    w_{1}^{(p,k)}
    \exp \left (
      -\eta \sum_{t=1}^{T}
       \mathcal{L}_t
       \left ( h_{\bm{\theta}_{t}}^{(p,k)} \right )
         \right )
    \end{aligned}
\end{equation}
\end{small}
Taking the logarithm on both side of Inequality~\eqref{eq:Apx11} and meanwhile considering $w_{1}^{(p,k)}=1/N$, we could obtain
{
\begin{small}
\begin{equation}
\label{eq:Apx12}
        \sum_{t=1}^{T}
        \sum_{q=1}^{N} 
        \bar{w}_{t}^{(p,q)} 
        \mathcal{L}_{t} 
        \left (
          h_{\bm{\theta}_{t}}^{(p,q)}
        \right )
        -
        \sum_{t=1}^{T}
        \mathcal{L}_t
        \left (
          h_{\bm{\theta}_{t}}^{(p,k)}
        \right )
        \leq
        \eta T + \frac{\log N}{\eta},
\end{equation}
\end{small}}
where we note that $\mathcal{L}_{t} ( h_{\bm{\theta}_{t}}^{(p,q)} ) \in [0,1]$ according to \textbf{AS2}, and Lemma~\ref{Lem2} has been proved, where we could divide $T$ on both side of the Inequality~\eqref{eq:Apx12}. 
\end{proof}

Lemma~\ref{Lem2} shows that the performance gap between an ensemble of collaborative predictors with the trainable weight $\bar{w}_{t}^{(p,q)}$ and an arbitrary collaborative predictor fed with any collaborative pairs $(v^{(p)}, v^{(k)})$, including the best predictor, can be upper bounded.
Combining Lemma~\ref{Lem1} and Lemma~\ref{Lem2}, we derive the upper bound of the regret as in Theorem~\ref{Them1}.
\begin{myThm}
\label{Them1}
Under {\bf AS1} and {\bf AS2}, and with the analysis of Lemma~\ref{Lem1} and Lemma~\ref{Lem2}, the regret of the proposed CoPU satisfies the following bound.
\begin{equation*}
\begin{aligned}
      & \frac{1}{T}
      \left (
      \sum_{t=1}^T
      \mathcal{L}_t
        \left (
          \sum_{q=1}^N
            \bar{w}_{t}^{(p,q)}h_{\bm{\theta}_{t}}^{(p,q)}
        \right )
    -
    \sum_{t=1}^T
      \mathcal{L}_t
        \left (
          h_{\bm{\theta}^*}^{(p,q^*)}
        \right ) 
      \right ) \\
    & \leq
    \frac{\log N}{\eta T} + 
    \frac{C_{\bm{\theta}}^2}{2 \eta T} + 
    \frac{\eta L^2}{2} + 
    \eta.
\end{aligned}
\end{equation*}
Setting $\eta = \mathcal{O} (1/\sqrt{T})$, the static regret converges with ${\rm R}_{T} = \mathcal{O} (1/\sqrt{T})$.  
\end{myThm}

\begin{figure*}[!tb]
    \centering
    \includegraphics[width=1.0\textwidth]{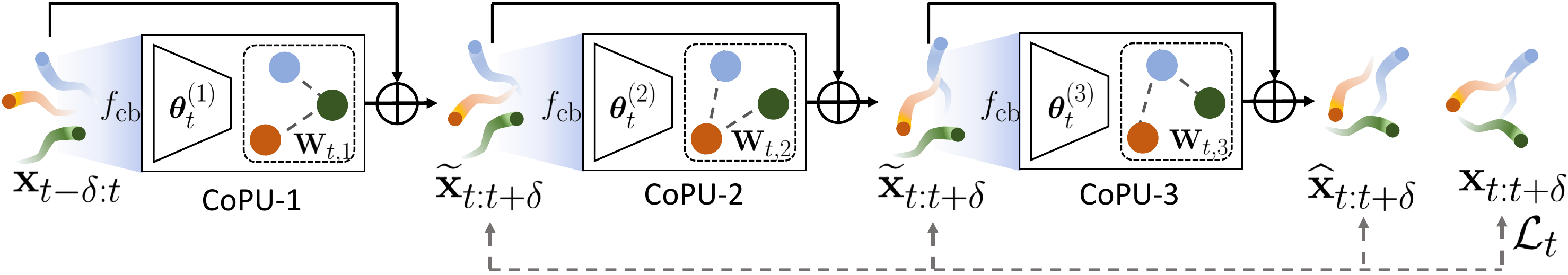}
    \caption{\small The architecture of the collaborative graph neural network (CoGNN). The model takes multiple time-series as inputs and aims to generate the future series. Each graph-temporal prediction unit (CoPU) in the model performs three operations for progressive prediction. The entire model is trained online for real-time prediction. Here we show three CoPUs to sketch the model architecture.}
    \label{fig:pipeline}
\end{figure*}

\begin{proof}
Combining Inequality~\eqref{eq:Apx6} and Inequality~\eqref{eq:Apx11}, we could obtain that
\begin{small}
\begin{equation}
    \begin{aligned}
    \label{eq:Apx13}
    & \frac{1}{T}
      \left (
      \sum_{t=1}^T
      \sum_{q=1}^N \bar{w}_{t}^{(p,q)}
      \mathcal{L}_t
        \left (
          h_{\bm{\theta}_{t}}^{(p,q)}
        \right )
    -
    \sum_{t=1}^T
      \mathcal{L}_t
        \left (
          h_{\bm{\theta}^*}^{(p,q^*)}
        \right ) 
      \right ) \\
    \leq ~ &
    \frac{\log N}{\eta T} + 
    \frac{C_{\bm{\theta}}^2}{2 \eta T} + 
    \frac{\eta L^2}{2} + 
    \eta.
    \end{aligned}
\end{equation}
\end{small}
Note that the loss function $\mathcal{L}_t(h_{\bm{\theta}_t}^{(p,q)})$ is convex w.r.t. any $h_{\bm{\theta}_t}^{(p,q)}$, thus the loss function $\mathcal{L}_t (\sum_{q=1}^N \bar{w}_{t}^{(p,q)}h_{\bm{\theta}_{t}}^{(p,q)} )$ has an upper bound $\sum_{q=1}^N \bar{w}_{t}^{(p,q)} \mathcal{L}_t ( h_{\bm{\theta}_{t}}^{(p,q)} )$; that is,
\begin{small}
\begin{equation}
    \label{eq:Apx14}
    \mathcal{L}_t 
    \left (
      \sum_{q=1}^N \bar{w}_{t}^{(p,q)}h_{\bm{\theta}_{t}}^{(p,q)} 
    \right )
    \leq
    \sum_{q=1}^N 
    \bar{w}_{t}^{(p,q)} 
    \mathcal{L}_t \left( h_{\bm{\theta}_{t}}^{(p,q)} \right)
\end{equation}
\end{small}
understanding that Inequality~\eqref{eq:Apx12} satisfies all the collaborative pairs in the online system, we could easily prove the Theorem~\ref{Them1} to represent the upper bound of the static regret. Plugging~\eqref{eq:Apx14} in~\eqref{eq:Apx13}, we could easily prove the Theorem 1.
\end{proof}

Therefore, the regret of CoPU is upper bounded during online training. Additionally, when $T$ becomes large, the regret converges to a very small value, reflecting the finite gap between the CoPU and the best single collaborative predictor. We also validate Theorem~\ref{Them1} by computing the regret along time in our experiment; see Fig.~\ref{fig:RegretPlot} (a).

\section{Collaborative Graph Neural Networks}
\label{sec:CoGNN}
Though one collaborative prediction unit (CoPU) can tackle the prediction task, in the experiments, we find that stacking multiple CoPUs can lead to even better performance. Therefore, we develop a deep architecture called \emph{collaborative graph neural networks} (CoGNN), which employs multiple CoPUs as computational building blocks.

\subsection{Network architecture}

We sketch the architecture of the proposed CoGNN in Fig.~\ref{fig:pipeline}, where we stack three CoPUs in cascade for example. Each CoPU in CoGNN progressively fine-tunes the predictions and approximates the target measurements step-by-step. According to the examplar Fig.~\ref{fig:pipeline}, the first and the second CoPUs produce the immediate prediction results $\widetilde{\x}_{t:t+\delta}$ for all the agents from the input clips or the first CoPU; the last CoPU generates the final prediction $\widehat{\x}_{t:t+\delta}$ from the results of the second CoPU.
For each CoPU, we apply a residual connection between the input and output, reflecting that the internal modules estimate the displacements from the input data to the target.

To train the network, we apply the same form of loss function on each CoPU, because each CoPU tackles a complete prediction task and could be trained directly; see the dash lines in gray in Fig.~\ref{fig:pipeline}. Moreover, the collaborative edge weights in each CoPU could be learned to directly optimize an explicit objective function, Eq.~\eqref{eq:w_update}, through the exponentiated update.

\begin{figure*}[!t]
  \begin{center}
    \begin{tabular}{cccc}
     \includegraphics[width=0.3\textwidth]{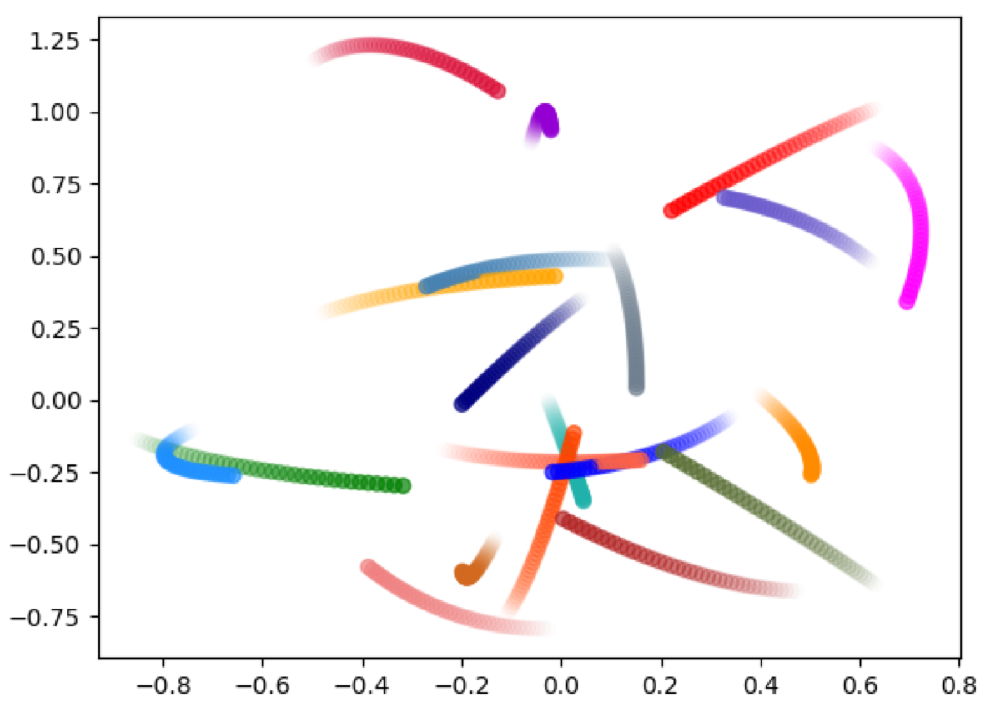}   
   & \includegraphics[width=0.3\textwidth]{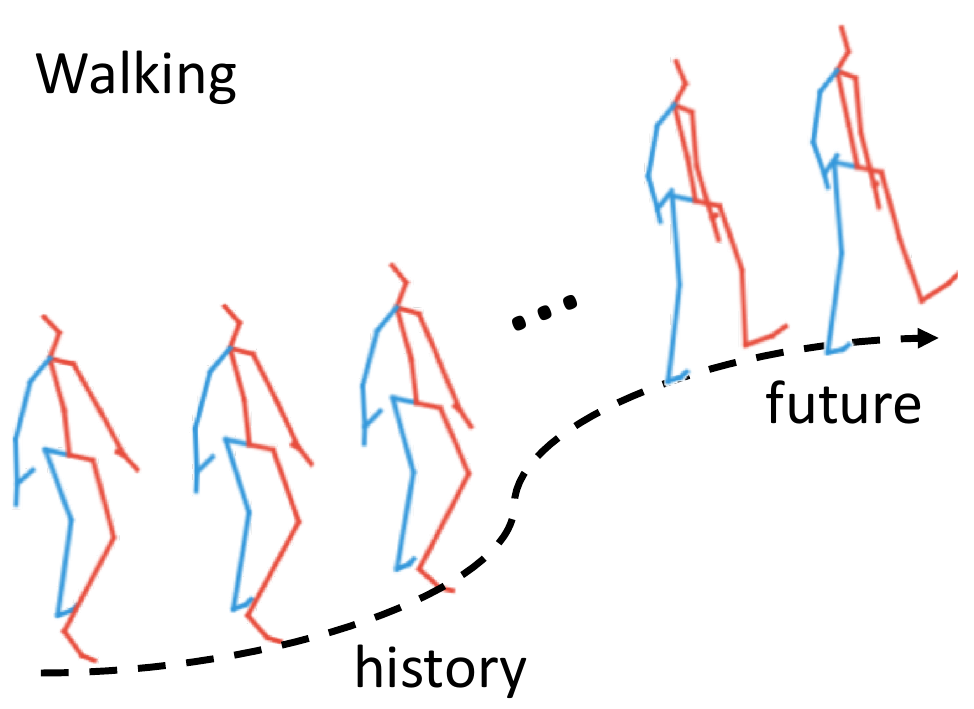}
   & \includegraphics[width=0.3\textwidth]{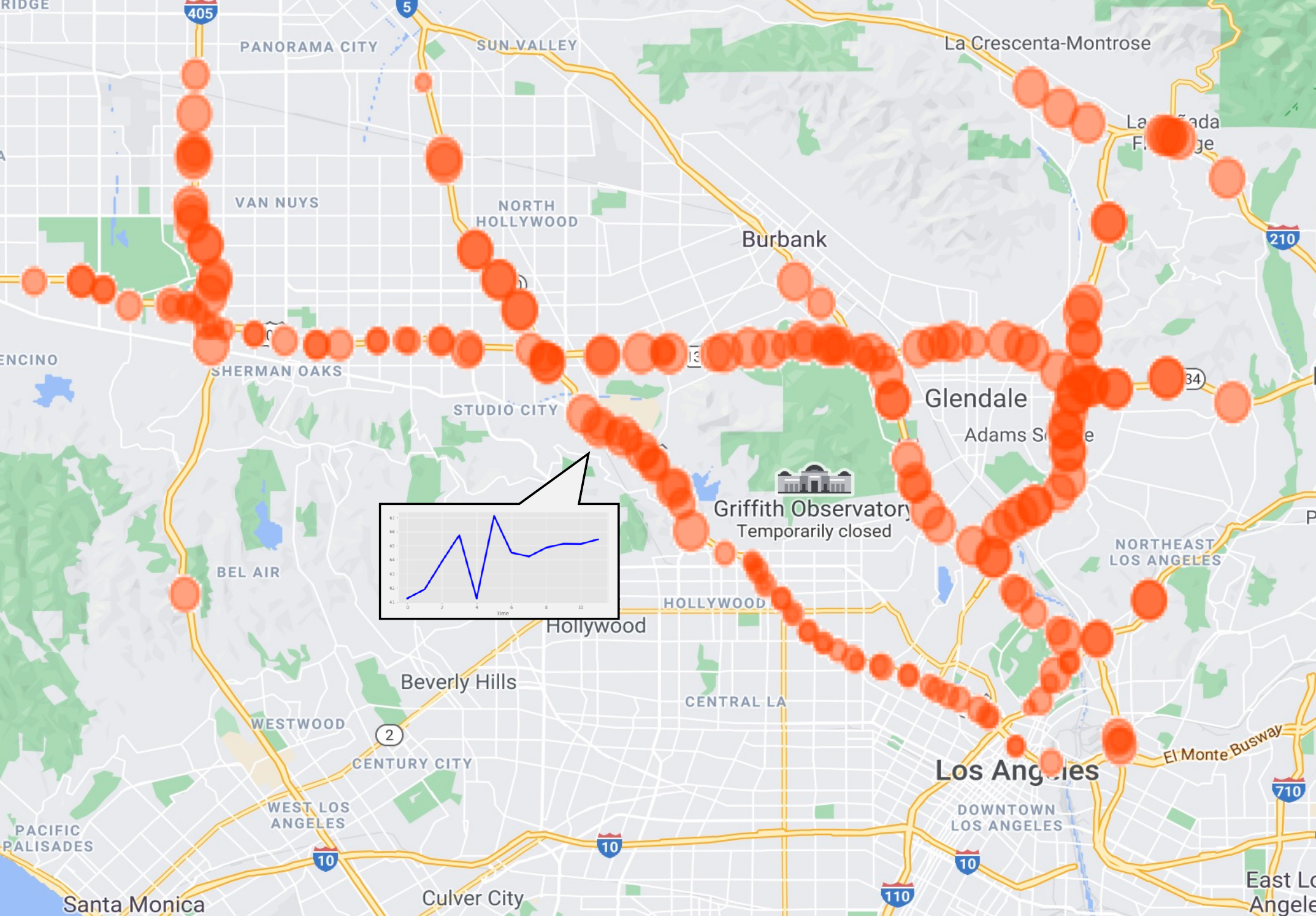}
    \\
    {\small (a) Simulated interaction system.} &  
    {\small (b) Human motion.} &
    {\small (c) Traffic speed distribution.} 
  \end{tabular}
\end{center}
\caption{\small Data examples of three scenarios of online multi-agent forecasting. (a) Simulated interaction systems: we generate systems based on 20 moving agents with various interactions, and we forecast the trajectory of each agent given the historical movements in an online setting. (b) Human motion: the human motion sequences consist of 3D skeleton-based human poses performing a certain action, e.g., walking, and we generate the future poses given the historical poses online. (c) Traffic speed distribution in cities: at different locations, there are sensors recording the traffic speed in real-time; that is, the traffic speed at each point is a continuous time-series (see the curve), while the larger sizes of the red circles denote the higher traffic speed at a certain time step. We aim to predict the traffic speed given the historical speed online.}
\label{fig:DataExample}
\end{figure*}

\subsection{Comparison with previous works}
Compared to previous graph-based prediction networks~\cite{Li_2020_CVPR_DMG, Shi_2019_CVPR_TSA, Yu_2018_IJCAI_STG,Kipf_2018_ICML_NRI, Graber_2020_CVPR, NEURIPS2020_e4d8163c}, our model contributes from two aspects. 

First,  to our best knowledge, the proposed CoGNN is the first online-learning method, which is learned and tested along with the streaming data and never stops. This is crucial because an online method could capture the highly dynamic patterns of the real-time series and reduce the training/testing gaps between the training and testing samples in practice. Our CoGNN is scalable and easy to train, which is suitable for online learning. However, previous methods only sample the multi-agent systems as sets of independent clips but neglect the real-time dynamics and complexity.

Second, the proposed CoGNN is more interpretable. The output of each layer in the proposed CoGNN predicts agents' future statuses in the original measurement space, reflecting a clear physical meaning. With a cascaded structure, our CoGNN can progressively fine-tune the time-series prediction, achieving effective and stable performance. However, previous methods use hierarchical graph propagation to extract hidden representations without interpretable gradual error reduction. 

\section{Experiments}
\label{sec:experiment}
To evaluate the proposed CoGNNs, we conduct extensive experiments on three important multi-agent forecasting tasks in an online setting; that is, online simulated trajectory prediction, online human motion prediction and online traffic speed prediction, where the corresponding scenarios and data examples are sketched in Fig.~\ref{fig:DataExample}.

In CoGNNs, we respectively use three forms of collaborative predictors $h_{\bm{\theta}_t}(\cdot)$ in CoPUs; see section~\ref{sec:model_design}. We note the three forms of models as `CoGNN-AR', `CoGNN-LSTM' and `CoGNN-TC', respectively.
In CoGNN-AR, to compute the state, ${\bf x}_{t-\delta:t}^{(p,q)}$, of any collaborative pair $(v^{(p)},v^{(q)})$, the order of measurement difference $D=10$; for CoGNN-LSTM and CoGNN-AR, the hidden dimension of $h_{\bm{\theta}_t}(\cdot)$ is $64$.
The entire CoGNNs are built with $2$ CoPUs, which enable precise and stable forecasting.
We train the models on one GTX-1080Ti GPU with PyTorch 1.4 framework. The learning rates are from $0.01$ to $0.075$ for different forms of CoGNNs. We also apply gradient clipping to constrain the absolute values of gradients within $10$ for robustness.

\subsection{Online simulated trajectory prediction}
We first evaluate the proposed CoGNNs in online simulated particle systems, which contains several moving particles whose positions and velocities are determined by their interactions. 

\mypar{Dataset}
In our simulated particle systems, there are $20$ particle agents, whose initial states (positions and velocities) are set randomly, moving in a 2D space with dynamic mutual interaction. The interactions among the agents could be represent as a graph structure with symmetric adjacency matrix. Here we use a simple interaction rule called 'spring' to construct the complex systems; that is, each agent has attractions with several other agents that are randomly selected. One example of our simulated system is illustrated in Fig.~\ref{fig:DataExample} (a). This system was similarly used by~\cite{Kipf_2018_ICML_NRI}, where only $5$ agents are set in a simple system. However, different from the previous work~\cite{Kipf_2018_ICML_NRI} which generates a set of independent trajectory sequences with fixed and known interactions, in our simulated systems, agents have randomly changed interactions with different agents to reflect dynamic and changing states online. In each interactive simulation, we change the interaction structures for $20$ times randomly; that is, the interaction adjacency matrix is randomly adjusted for $20$ times. Moreover, we conduct $10$ times of different simulation to demonstrate the effectiveness and robustness. In our experiments, along the streaming data, our CoGNNs are fed with each 10-frame clip one-by-one along time and generate the next 10-frame clip for prediction.

\mypar{Baselines and evaluation metrics}
We compare our CoGNNs with several methods on trajectory prediction, including  NRI~\cite{Kipf_2018_ICML_NRI}, DNRI~\cite{Graber_2020_CVPR} and EvolveGraph~\cite{NEURIPS2020_e4d8163c}. These methods are designed for the offline scenarios; thus we modify these works to be run online with the same inputs and outputs as our CoGNNs. We also adjust their hyper-parameters to be reasonable.
Besides being compared to the state-of-the-art works, CoGNNs are also compared to two degenerated variants: 1) CoGNNs forecasting without collaborative graphs, i.e., `CoGNNs (no ${\bf W}_t$)'; and 2) CoGNNs training collaborative graphs via online gradient descent end-to-end in black-boxes, i.e., `CoGNNs (e2e ${\bf W}_t$)'.
To evaluate various methods, we employ the average mean square errors (MSE) between the predicted locations and ground-truth of all the agents over the training/testing process.

\begin{table*}[t]
    \centering
    \footnotesize
    \caption{\small Prediction results (mean values and standard deviations of MSEs) on the tasks of online simulated prediction, where we consider 10 interaction systems with 20 trajectories. We also compare the variants of CoGNNs: CoGNNs (no ${\bf W}_t$), which predict each agent based on this agent itself without using the collaborative graph; and CoGNNs (e2e ${\bf W}_t$), which train the collaborative graph through standard gradient descent in an end-to-end setting.}
    \renewcommand{\arraystretch}{1.0}
    \resizebox{0.74\textwidth}{!}{
    \begin{tabular}{c|c|ccccc}
        \hline
        \hline
        \multicolumn{2}{c|}{~} & \multicolumn{5}{c}{Prediction steps}\\
        \hline
        \multicolumn{2}{c|}{Methods} & 1 & 2 & 5 & 8 & 10 \\
        \hline
        \multirow{4}{*}{\rotatebox{90}{Baselines}} & ZeroV & 13.43 $\pm$ 1.94 & 
                                                             13.43 $\pm$ 1.95 &
                                                             13.48 $\pm$ 2.08 & 
                                                             13.58 $\pm$ 2.35 & 
                                                             13.69 $\pm$ 2.64 \\
        ~ & NRI~\cite{Kipf_2018_ICML_NRI} & 0.56 $\pm$ 0.16 &
                                            0.58 $\pm$ 0.15 &
                                            0.60 $\pm$ 0.18 &
                                            1.03 $\pm$ 0.30 &
                                            1.59 $\pm$ 0.46 \\
        ~ & DNRI~\cite{Graber_2020_CVPR} & 0.26 $\pm$ 0.10 &
                                           0.28 $\pm$ 0.12 &
                                           0.31 $\pm$ 0.12 &
                                           0.51 $\pm$ 0.15 &
                                           0.77 $\pm$ 0.19 \\
        ~ & EvolveGraph~\cite{NEURIPS2020_e4d8163c} & 0.22 $\pm$ 0.10 &
                                                      0.24 $\pm$ 0.09 &
                                                      0.28 $\pm$ 0.10 &
                                                      0.44 $\pm$ 0.13 &
                                                      0.68 $\pm$ 0.15 \\
        \hline
        \multirow{3}{*}{\rotatebox{90}{no ${\bf W}_t$}} 
          & CoGNN-AR (no ${\bf W}_t$) & 0.23 $\pm$ 0.06 &
                                        0.24 $\pm$ 0.07 & 
                                        0.41 $\pm$ 0.12 & 
                                        0.75 $\pm$ 0.23 &
                                        1.09 $\pm$ 0.32 \\
        ~ & CoGNN-LSTM (no ${\bf W}_t$) & 0.16 $\pm$ 0.04 &
                                          0.16 $\pm$ 0.05 & 
                                          0.20 $\pm$ 0.05 & 
                                          0.40 $\pm$ 0.14 &
                                          0.61 $\pm$ 0.16 \\
        ~ & CoGNN-TC (no ${\bf W}_t$) & 0.19 $\pm$ 0.08 &
                                        0.23 $\pm$ 0.11 & 
                                        0.24 $\pm$ 0.11 & 
                                        0.42 $\pm$ 0.16 &
                                        0.63 $\pm$ 0.18 \\
        \hline
        \multirow{3}{*}{\rotatebox{90}{e2e ${\bf W}_t$}} 
          & CoGNN-AR (e2e ${\bf W}_t$) & 0.20 $\pm$ 0.04 &
                                         0.21 $\pm$ 0.06 & 
                                         0.37 $\pm$ 0.15 & 
                                         0.74 $\pm$ 0.22 &
                                         1.06 $\pm$ 0.34 \\
        ~ & CoGNN-LSTM (e2e ${\bf W}_t$) & 0.17 $\pm$ 0.06 &
                                           0.16 $\pm$ 0.06 & 
                                           0.19 $\pm$ 0.07 & 
                                           0.37 $\pm$ 0.13 &
                                           0.60 $\pm$ 0.18 \\
        ~ & CoGNN-TC (e2e ${\bf W}_t$) & {\bf 0.15} $\pm$ 0.05 &
                                         0.18 $\pm$ 0.07 & 
                                         0.20 $\pm$ 0.07 & 
                                         0.36 $\pm$ 0.12 &
                                         0.58 $\pm$ 0.17 \\
        \hline
        \multirow{3}{*}{\rotatebox{90}{Ours}} 
          & CoGNN-AR & 0.19 $\pm$ 0.06 &
                       0.20 $\pm$ 0.06 & 
                       0.37 $\pm$ 0.12 & 
                       0.70 $\pm$ 0.21 &
                       1.01 $\pm$ 0.30  \\
        ~ & CoGNN-LSTM & {\bf 0.15} $\pm$ 0.05 &
                         {\bf 0.15} $\pm$ 0.06 & 
                         0.19 $\pm$ 0.05 & 
                         {\bf 0.35} $\pm$ 0.10 &
                         0.59 $\pm$ 0.17  \\
        ~ & CoGNN-TC & 0.16 $\pm$ 0.06 &
                       0.18 $\pm$ 0.05 & 
                       {\bf 0.18} $\pm$ 0.06 & 
                       {\bf 0.35} $\pm$ 0.12 &
                       {\bf 0.56} $\pm$ 0.16  \\
        \hline
        \hline
    \end{tabular}}
    \label{tab:result_on_simulation}
\end{table*}

\mypar{Results}
For various methods, Table~\ref{tab:result_on_simulation} presents the average forecasting results over the online learning processes, where we show the prediction performance at $5$ prediction steps of the generated 10-frame clips, respectively. Since all the models are run on $10$ different interaction systems,  the mean values and standard deviations of the prediction MSEs are presented for comparison. 
We see that, 
1) compared to the state-of-the-art methods, our CoGNNs achieve more precise forecasting and outperform the baselines, especially CoGNN-LSTM and CoGNN-TC achieve much more effective prediction at various prediction steps; 
2) compared to the model variants, `no ${\bf W}_t$' and `e2e ${\bf W}_t$', the proposed CoGNNs consistently outperforms these two variants, demonstrating the importance and effectiveness of our intepretable collaborative graphs and the exponentiated update strategy.

\subsection{Online human motion prediction}
We next conduct experiments on online human motion prediction to verify the effectiveness of our CoGNN.

\mypar{Dataset} Various models are trained and tested on two motion capture datasets: Human3.6M (H3.6M)~\cite{Ionescu_2014_H36m} and CMU Mocap\footnote{http://mocap.cs.cmu.edu/}.
H3.6M contains $15$ classes of activities. There are $32$ body joints with their 3D coordinates.
An examplar motion clip in H3.6M is shown in Fig.~\ref{fig:DataExample} (b).
CMU Mocap contains $8$ activities with $38$ joints.
Following the previous settings~\cite{Martinez_2017_CVPR_OHM, Li_2018_CVPR_CSS, Mao_2019_ICCV_LTD}, we downsample all motion sequences along time by two for both datasets.
We take the motion clips of the past $400$ms ($10$ frames) as input, and forecast the future $400$ms.
\begin{table*}[tb]
    \centering
    \footnotesize
    \caption{\small Average MPJPEs of CoGNNs and baselines for motion prediction on 4 representative actions of H3.6M. We also compare two types of variants of CoGNNs, i.e., `no ${\bf W}_t$' and `e2e ${\bf W}_t$'; see the definition of these two variants in the caption of Table~\ref{tab:result_on_simulation}.}
    \resizebox{1.0\textwidth}{!}{
    \begin{tabular}{c|c|cccc|cccc|cccc|cccc}
        \hline
        \hline
        \multicolumn{2}{c|}{Motions} & \multicolumn{4}{|c|}{Walking} & \multicolumn{4}{|c|}{Eating} & \multicolumn{4}{|c|}{Smoking} & \multicolumn{4}{|c}{Discussion}\\
        \hline
        \multicolumn{2}{c|}{Methods} & 80&160&320&400 & 80&160&320&400 & 80&160&320&400 & 80&160&320&400 \\
        \hline
        \multirow{8}{*}{\rotatebox{90}{Baselines}}
        & Res-sup~\cite{Martinez_2017_CVPR_OHM} & 23.0 & 32.1 & 37.8 & 44.4 & 20.6 & 30.2 & 37.1 & 46.0 & 23.2 & 27.6 & 35.8 & 42.7 & 30.1 & 41.8 & 46.8 & 52.0 \\ 
        ~ & NRI~\cite{Kipf_2018_ICML_NRI} & 21.6 & 30.4 & 36.0 & 42.4 & 19.5 & 28.7 & 36.4 & 42.5 & 22.4 & 26.8 & 33.2 & 40.1 & 28.7 & 39.3 & 44.9 & 51.3\\
        ~ & CSM~\cite{Li_2018_CVPR_CSS} & 22.1 & 31.5 & 39.4 & 43.3 & 19.9 & 29.1 & 39.7 & 43.9 & 21.3 & 26.5 & 32.4 & 38.6 & 31.2 & 38.0 & 44.4 & 50.8 \\
        ~ & Traj-GCN~\cite{Mao_2019_ICCV_LTD} & 22.9 & 30.1 & 38.9 & 43.6 & 19.0 & 24.9 & 35.8 & 38.6 & 19.7 & 25.3 & 28.2 & 35.9 & 24.1 & 36.2 & 43.4 & 47.3 \\
        ~ & DNRI~\cite{Graber_2020_CVPR} & 22.3 & 31.8 & 39.2 & 43.7 & 19.3 & 27.6 & 36.7 & 40.5 & 20.5 & 26.3 & 29.1 & 36.7 & 26.8 & 37.2 & 43.8 & 49.0 \\
        ~ & DMGNN~\cite{Li_2020_CVPR_DMG} & 22.6 & 31.1 & 39.2 & 44.1 & 18.2 & 24.2 & 32.5 & 39.2 & 19.9 & 24.3 & 27.8 & 35.7 & 24.6 & 35.8 & 43.4 & 49.5 \\
        ~ & HisRep~\cite{Mao_2020_ECCV_HRI} & 22.6 & 29.0 & 37.3 & 42.1 & 17.3 & 22.5 & 33.6 & 37.3 & 19.2 & 25.4 & 29.9 & 36.3 & 22.8 & 33.6 & 40.7 & 46.5 \\
        ~ & EvolveGraph~\cite{NEURIPS2020_e4d8163c} & 22.2 & 30.7 & 38.3 & 43.5 & 18.1 & 23.8 & 34.3 & 38.9 & 20.4 & 26.0 & 29.7 & 37.1 & 24.2 & 34.7 & 42.9 & 48.4\\
        \hline
        \multirow{3}{*}{\rotatebox{90}{no ${\bf W}_t$}}
        & CoGNN-AR (no ${\bf W}_t$) & 26.2 & 39.9 & 62.3 & 65.4 & 20.3 & 27.8 & 43.0 & 48.7 & 23.2 & 30.5 & 46.0 & 52.7 & 24.8 & 35.6 & 55.7 & 63.7 \\
        ~ & CoGNN-LSTM (no ${\bf W}_t$) & 23.3 & 23.9 & {\bf 29.8} & 40.1 & 19.8 & 20.0 & 26.4 & 35.0 & 20.3 & 20.8 & 28.0 & 33.9 & 29.5 & 30.0 & 40.6 & 48.9 \\
        ~ & CoGNN-TC (no ${\bf W}_t$) & 25.8 & 27.3 & 37.2 & 47.2 & 19.6 & 20.1 & 30.1 & 38.3 & 17.8 & 19.1 & 29.7 & 36.3 & 22.2 & 25.3 & 36.7 & 48.3 \\
        \hline
        \multirow{3}{*}{\rotatebox{90}{e2e ${\bf W}_t$}}
        & CoGNN-AR (e2e ${\bf W}_t$) & 24.9 & 35.1 & 50.5 & 57.9 & 17.6 & 26.8 & 38.7 & 40.3 & 22.3 & 28.2 & 43.6 & 48.7 & 24.9 & 34.3 & 51.7 & 58.6 \\
        ~ & CoGNN-LSTM (e2e ${\bf W}_t$) & 23.5 & 23.7 & 33.7 & 40.9 & 18.4 & 17.0 & 27.6 & 34.1 & 19.4 & 20.0 & 27.3 & 24.2 & 28.4 & 29.6 & 40.1 & 47.5 \\
        ~ & CoGNN-TC (e2e ${\bf W}_t$) & 22.6 & 25.2 & 35.3 & 41.7 & 18.3 & 19.9 & 29.2 & 35.7 & 18.0 & 18.8 & 28.4 & 35.5 & 22.0 & 26.7 & 36.1 & 47.5 \\
        \hline
        \multirow{3}{*}{\rotatebox{90}{Ours}}
        & CoGNN-AR & 22.3 & 33.7 & 51.5 & 56.8 & {\bf 16.6} & 23.0 & 36.5 & 41.5 & 19.5 & 27.7 & 41.6 & 45.4 & 24.1 & 30.1 & 50.4 & 58.0 \\
        ~ & CoGNN-LSTM & 22.6 & 23.3 & 33.6 & 40.7 & 18.2 & 19.2 & {\bf 26.1} & {\bf 33.7} & {19.7} & 19.4 & {\bf 25.2} & 33.8 & 28.7 & 29.3 & 38.9 & 47.3 \\
        ~ & CoGNN-TC & {\bf 21.5} & {\bf 21.4} & {31.0} & {\bf 39.7} & {16.8} & {\bf 16.2} & {28.7} & {36.0} & {\bf 16.4} & {\bf 16.6} & {26.6} & {\bf 33.6} & {\bf 19.9} & {\bf 18.5} & {\bf 34.1} & {\bf 45.2} \\
        \hline
        \hline
    \end{tabular}
    }
    \label{tab:result_on_H36M}
\end{table*}

\begin{table}[tb]
    \centering
    \footnotesize
    \caption{\small Average MPJPEs of CoGNNs and baselines for motion prediction on CMU Mocap. We also compare two types of variants of CoGNNs, i.e., `no ${\bf W}_t$' and `e2e ${\bf W}_t$'; see the definition of these two variants in the caption of Table~\ref{tab:result_on_simulation}}
    \renewcommand{\arraystretch}{1.0}
    \resizebox{1.0\columnwidth}{!}{
    \begin{tabular}{c|c|cccc}
        \hline
        \hline
        \multicolumn{2}{c|}{Motions} & \multicolumn{4}{|c}{Average} \\
        \hline
        \multicolumn{2}{c|}{Methods}  & 80&160&320&400  \\
        \hline
        \multirow{8}{*}{\rotatebox{90}{Baselines}}
        & Res-sup~\cite{Martinez_2017_CVPR_OHM} & 22.9 & 30.7 & 40.4 & 47.4  \\ 
        ~ & CSM~\cite{Li_2018_CVPR_CSS} & 21.0 & 28.5 & 35.6 & 41.7 \\
        ~ & NRI~\cite{Kipf_2018_ICML_NRI} & 19.7 & 26.7 & 33.6 & 41.1 \\
        ~ & Traj-GCN~\cite{Mao_2019_ICCV_LTD} & 17.3 & 24.5 & 31.8 & 37.9 \\
        ~ & DNRI~\cite{Graber_2020_CVPR} & 18.3 & 25.5 & 32.4 & 38.9 \\
        ~ & DMGNN~\cite{Li_2020_CVPR_DMG} & 16.8 & 23.8 & 30.6 & 37.7 \\
        ~ & HisRep~\cite{Mao_2020_ECCV_HRI} & 15.6 & 21.9 & 28.7 & 35.2 \\
        ~ & EvolveGraph~\cite{NEURIPS2020_e4d8163c} & 17.0 & 24.1 & 33.4 & 38.6\\
        \hline
        \multirow{3}{*}{\rotatebox{90}{no ${\bf W}_t$}}
        & CoGNN-AR (no ${\bf W}_t$) & 27.1 & 34.7 & 52.7 & 57.0 \\
        ~ & CoGNN-LSTM (no ${\bf W}_t$) & 18.6 & 21.0 & 30.0 & 37.2 \\
        ~ & CoGNN-TC (no ${\bf W}_t$) & 17.0 & 19.1 & 28.4 & 34.7 \\
        \hline
        \multirow{3}{*}{\rotatebox{90}{e2e ${\bf W}_t$}}
        & CoGNN-AR (e2e ${\bf W}_t$) & 23.1 & 30.6 & 45.4 & 51.2 \\
        ~ & CoGNN-LSTM (e2e ${\bf W}_t$) & 16.9 & 18.7 & 24.6 & 30.2 \\
        ~ & CoGNN-TC (e2e ${\bf W}_t$) & 16.4 & 18.2 & 25.4 & 31.4 \\
        \hline
        \multirow{3}{*}{\rotatebox{90}{Ours}}
         & CoGNN-AR & 21.6 & 28.5 & 43.5 & 50.3 \\
        ~ & CoGNN-LSTM & 16.7 & 17.7 & {\bf 23.3} & 29.0 \\
        ~ & CoGNN-TC & {\bf 15.4} & {\bf 17.2} & 23.8 & {\bf 28.5} \\
        \hline
        \hline
    \end{tabular}
    }
    \label{tab:result_on_CMU}
\end{table}

\mypar{Baselines and evaluation metrics}
We compare our model to state-of-the-art works, including two non-graph methods, Res-sup~\cite{Martinez_2017_CVPR_OHM} and CSM~\cite{Li_2018_CVPR_CSS}, and six graph-based methods, NRI~\cite{Kipf_2018_ICML_NRI}, Traj-GCN~\cite{Mao_2019_ICCV_LTD},  DNRI~\cite{Graber_2020_CVPR}, DMGNN~\cite{Li_2020_CVPR_DMG}, HisRep~\cite{Mao_2020_ECCV_HRI} and EvolveGraph~\cite{NEURIPS2020_e4d8163c}.
For these offline methods, we also adjust them to take the same streaming data as our CoGNNs and generate future poses as effectively as possible in an online setting.
For evaluation, we calculate the Mean Per Joint Postion Error (MPJPE) between each predicted pose and the corresponding ground-truth.

\mypar{Results}
We first evaluate various methods on H3.6M.
Besides being compared to the state-of-the-art works, CoGNNs are also compared to two degenerated variants: i.e., `CoGNN (no ${\bf W}_t$)' and `CoGNN (e2e ${\bf W}_t$)'; see introduction in the experiments of online simulated trajectory prediction. 
Table~\ref{tab:result_on_H36M} shows the forecasting MPJPEs on 4 representative actions: ‘Walking’, ‘Eating’, ‘Smoking’ and ‘Discussion’, presenting the average MPJPEs at the future $80$ms, $160$ms, $320$ms and $400$ms.
We see that, 
1) compared to the state-of-the-art methods, our CoGNNs achieve more precise forecasting and outperform the baselines; 
2) compared to the model variants, `no ${\bf W}_t$' and `e2e ${\bf W}_t$', the proposed CoGNNs consistently outperforms these two variants, demonstrating the importance and effectiveness of our intepretable collaborative graphs and the exponentiated update strategy. 
See the forecasting results of another 11 actions in Appendix.

To further qualitatively evaluate our CoGNNs, we visualize the future poses produced by various methods for human motion prediction on H3.6M dataset. We illustrate the motions of a clip of `Walking' within the future 400 ms, which are predicted by DMGNN~\cite{Li_2020_CVPR_DMG}, CoGNN-LSTM and CoGNN-TC; see Fig.~\ref{fig:PredPlot}.
We see that, for the baseline model, DMGNN, there are large errors after the 160th ms of the prediction time, where the left arm on the human body cannot bend flexibly. As for the CoGNN-LSTM and CoGNN-TC, the generated poses are much closet to the ground-truths than the DMGNN model. Moreover, we note that, the illustrated human poses is the predicted samples at one certain iteration step. To evaluate the online prediction system, we should focus more on the continuous performance and the prediction errors during the whole online learning process.

Additionally, We also various methods on another large-scale datatset, CMU Mocap, for online motion prediction. The average prediction results (MPJPE) across all the actions at the future $80$ms $160$ms $320$ms and $400$ms are presented in Table~\ref{tab:result_on_CMU}.
We see that, 
1) the proposed CoGNN models outperform the state-of-the-art methods with a large margin;
2) the proposed versions of CoGNNs achieve consistent better performances than the other two types of model variants to show the effectiveness of our collaborative mechanism.
The detailed forecasting results of each class of activities are presented in Appendix.

\begin{figure}[tb!]
    \centering
    \includegraphics[width=0.98\columnwidth]{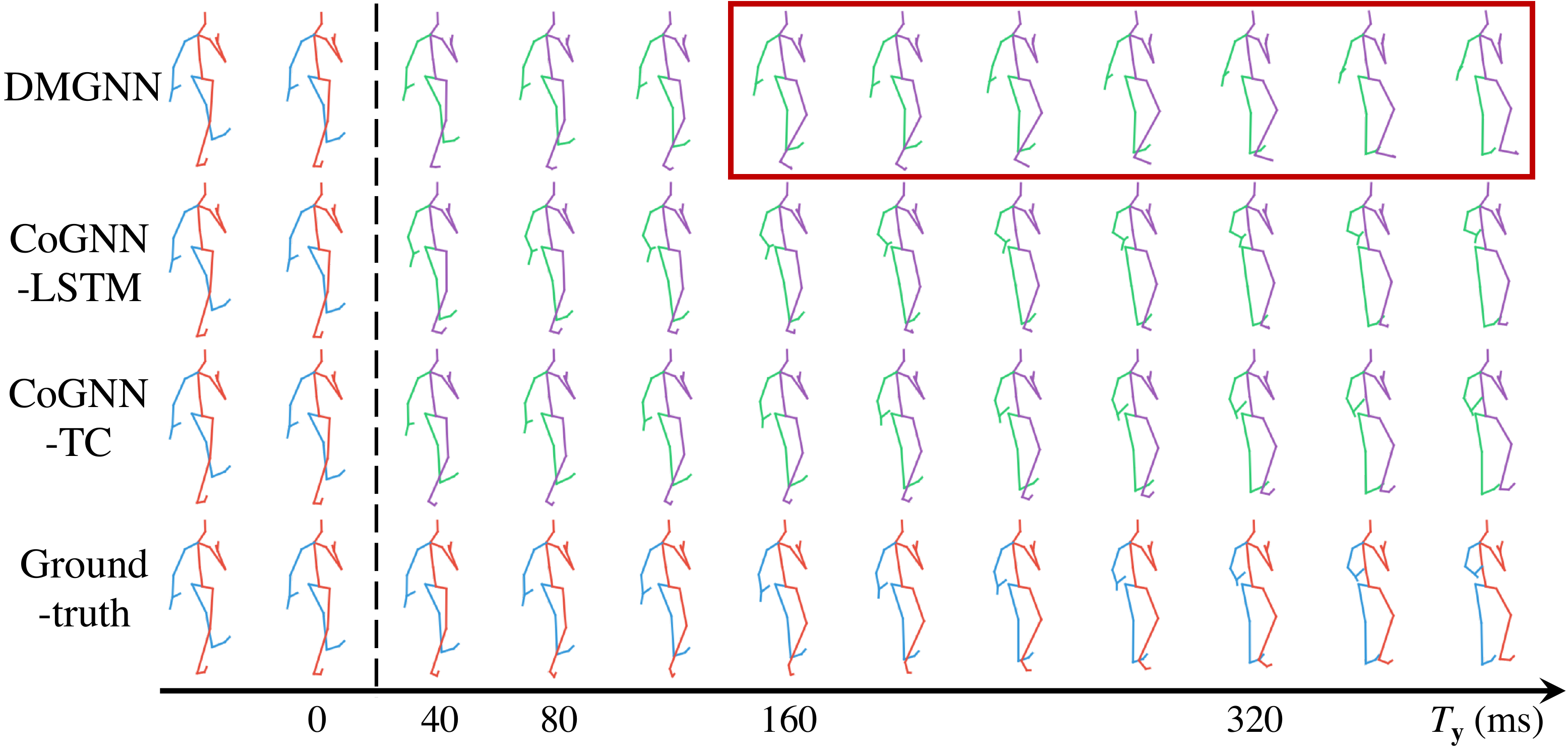}
    \caption{\small Predicted samples for human motion prediction on the action of `Walking' on H3.6M dataset.}
    \label{fig:PredPlot}
\end{figure}

\begin{table*}[tb]
    \centering
    \footnotesize
    \caption{\small Comparisons between our CoGNNs and state-of-the-art methods for online traffic speed forecasting on four datasets based on three metrics. We also present the degenerated variants of CoGNNs which abandon the collaborative graphs, i.e., `no ${\bf W}_t$'.}
    \resizebox{1.0\textwidth}{!}{
    \begin{tabular}{c|c|ccc|ccc|ccc|ccc}
        \hline
        \hline
        \multicolumn{2}{c|}{Datasets} & \multicolumn{3}{|c|}{METR-LA} & \multicolumn{3}{|c}{PeMS-BAY} & \multicolumn{3}{|c}{PeMS-D4} & \multicolumn{3}{|c}{PeMS-D8}\\
        \hline
        \multicolumn{2}{c|}{Methods} & MAE & RMSE & MAPE &  MAE & RMSE & MAPE & MAE & RMSE & MAPE &  MAE & RMSE & MAP\\
        \hline
        \multirow{6}{*}{\rotatebox{90}{Baselines}} & 
         ST-GCN~\cite{Yu_2018_IJCAI_STG} & 5.68 & 8.92 & 13.0\% & 3.09 & 6.57 & 7.7\% & 37.28 & 46.44 & 30.5\% & 30.49 & 39.57 & 24.8\% \\
         ~ & DCRNN~\cite{Li_2018_ICLR_Diffusion} & 5.45 & 8.76 & 12.6\% & 3.07 & 6.69 & 7.8\% & 37.67 & 46.95 & 31.1\% & 31.31 & 39.35 & 25.4\% \\
         ~ & ASTGCN~\cite{Guo_2019_AAAI_ABS} & 4.66 & 7.43 & 11.7\% & 2.51 & 4.29 & 6.0\% & 32.15 & 43.72 & 28.8\% & 28.85 & 36.97 & 23.0\% \\
         ~ & G-WaveNet~\cite{Wu_2019_IJCAI_GWD} & 4.42 & 7.10 & 11.2\% & 2.31 & 4.07 & 5.3\% & 28.58 & 39.59 & 25.6\% & 27.81 & 36.03 & 19.8\% \\
         ~ & STSGCN~\cite{Song_2020_AAAI_STS} & 4.73 & 7.62 & 11.6\% & 2.44 & 4.31 & 6.8\% & 23.82 & 30.85 & 18.9\% & 24.47 & 33.11 & 16.4\% \\
         ~ & AGCRN~\cite{Bai_2020_NIPS_AGC} & 4.19 & 6.40 & 8.3\% & 2.16 & 3.74 & 4.7\% & 19.65 & 28.54 & 16.0\% & 22.72 & 29.03 & 13.4\% \\
         \hline
         \multirow{3}{*}{\rotatebox{90}{no ${\bf W}_t$}} & 
         CoGNN-AR (no ${\bf W}_t$) & 4.42 & 6.61 & 8.8\% & 1.72 & 2.65 & 4.1\% & 19.37 & 28.02 & 17.3\% & 16.92 & 26.61 & 12.7\%\\
         ~ & CoGNN-LSTM (no ${\bf W}_t$) & 6.08 & 8.35 & 14.6\% & 2.74 & 3.68 & 4.4\%& 27.45 & 35.14 & 19.3\% & 20.12 & 29.73 & 13.7\%\\
         ~ & CoGNN-TC (no ${\bf W}_t$) & 3.61 & 5.89 & 8.0\% & 1.43 & 2.41 & 3.0\% & 20.27 & 29.48 & 16.1\% & 16.99 & 24.14 & 11.9\%\\
         \hline
         \multirow{3}{*}{\rotatebox{90}{Ours}} & 
         CoGNN-AR & 4.09 & 6.22 & 8.3\% & 1.59 & 2.42 & 3.8\% & 19.07 & {\bf 26.94} & 15.3\% & 15.75 & 23.74 & 11.4\%\\
         ~ & CoGNN-LSTM & 5.04 & 7.28 & 13.4\% & 2.21 & 3.52 & 4.2\% & 26.95 & 32.31 & 18.9\% & 19.62 & 28.87 & 13.1\%\\
         ~ & CoGNN-TC & {\bf 3.46} & {\bf 5.75} & {\bf 7.8\%} & {\bf 1.36} & {\bf 2.24} & {\bf 2.8\%} & {\bf 18.93} & 27.71 & {\bf 14.7\%} & {\bf 15.51} & {\bf 22.35} & {\bf 10.8\%}\\
        \hline
        \hline
    \end{tabular}
    }
    \label{tab:result_on_traffic}
\end{table*}

\subsection{Online traffic speed prediction}
Furthermore, the proposed CoGNNs are also evaluated on the tasks of online traffic speed prediction.

\mypar{Datasets}
We conduct experiments on four large-scale traffic datasets in the real world, including the traffic data of Los Angeles county~\cite{Jagadish_2014_Commu_BDI} or different regions of the California highway~\cite{Chen_2011_TRR_FPM}:

$\bullet$ METR-LA contains traffic speed collected by 207 sensors around Los Angeles from March to June in 2012.

$\bullet$  PeMS-BAY contains traffic speed at Bay Area collected by 325 sensors on CalTrans Performance Measurement System (PeMS) from Jan. to May in 2017.

$\bullet$ PeMS-D4 contains traffic speed collected by 307 sensors at San Francisco from Jan. to Feb. in 2018.

$\bullet$ PeMS-D8 contains traffic speed collected by 170 sensors in San Bernardino from July to Aug. in 2016.

We illustrate the traffic speed distribution of METR-LA in Fig.~\ref{fig:DataExample} (c) as an example, where we denote that each sensor captures the dynamic traffic speed along time.
For these datasets, we aggregate the data sequence into 5-minute intervals from 30-second frequency.
We take the past $12$-frame-length ($1$ hour) sequence as the input and forecast the next $12$ frames.

\mypar{Baselines and evaluation metrics}
To evaluate the effectiveness our CoGNNs on the online traffic speed prediction, we introduce several state-of-the-art models for comparison,including 
ST-GCN~\cite{Yu_2018_IJCAI_STG}, DCRNN~\cite{Li_2018_ICLR_Diffusion}, ASTGCN~\cite{Guo_2019_AAAI_ABS}, G-WaveNet~\cite{Wu_2019_IJCAI_GWD}, STSGCN~\cite{Song_2020_AAAI_STS} and AGCRN~\cite{Bai_2020_NIPS_AGC}.
To test various methods, we utilize three metrics, which are mean absolute error (MAE), root mean square error (RMSE) and mean absolute percentage error (MAPE) over the whole iterative training and test process.

\mypar{Results}
We compare CoGNNs to various algorithms on four datasets.
We also test the degenerated variants of CoGNNs without collaborative graphs, i.e., CoGNNs (no ${\bf W}_t$).
We evaluate on the mean MAE, mean RMSE and mean MAPE as shown in Table~\ref{tab:result_on_traffic}.
We see that, 
1) compared to the previous works, our CoGNNs obtain the lowest errors on all the datasets;
2) compared to the model variants without collaborative graphs, the complete CoGNNs also achieve the much better prediction performances.
The detailed traffic forecasting at different prediction time are presented in Appendix.

\subsection{Model analysis}
\mypar{Regrets of various forms of CoGNNs}
\begin{figure}[tb]
    \begin{center}
        \begin{tabular}{cc}
            \includegraphics[width=0.45\columnwidth]{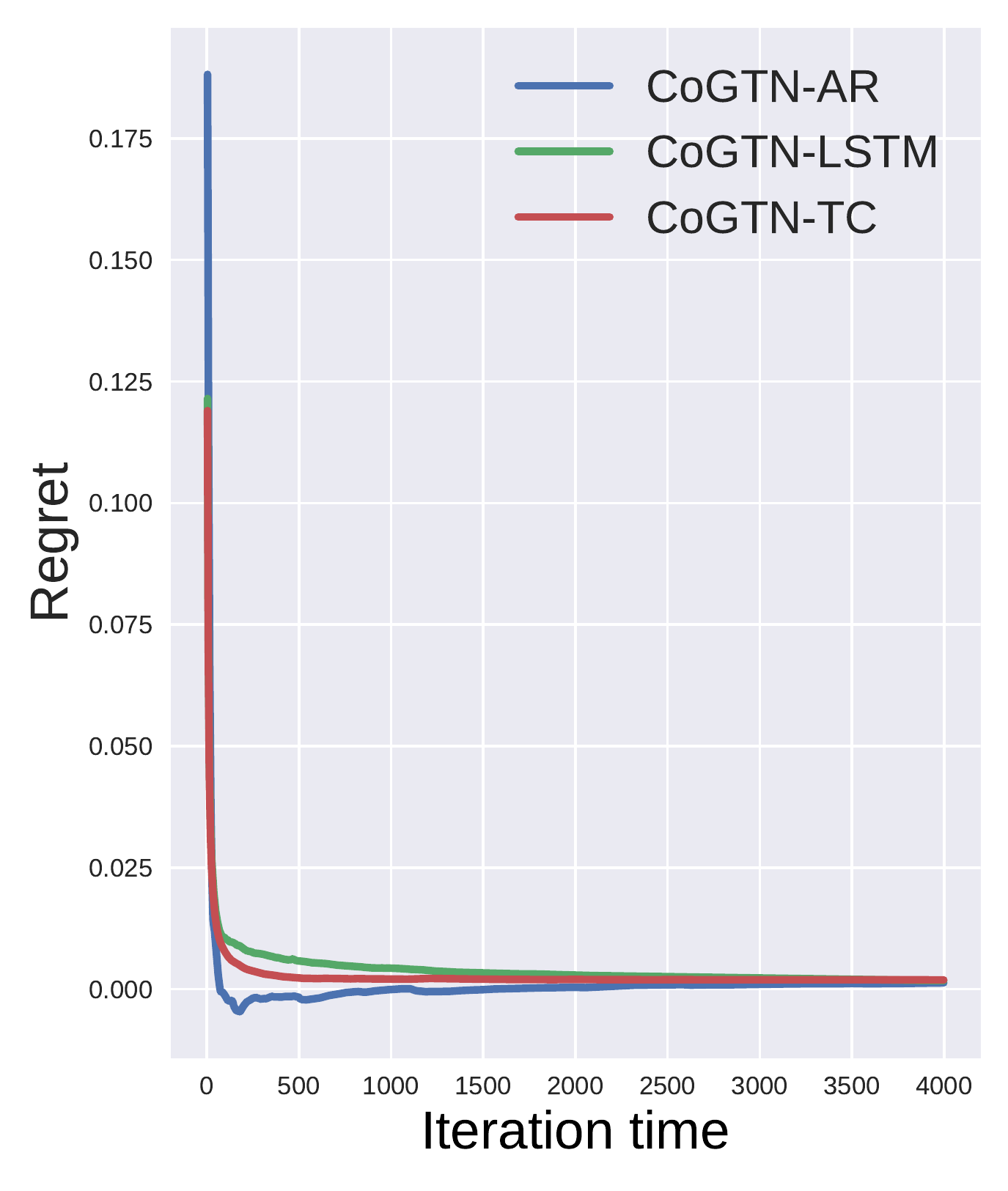} & 
            \includegraphics[width=0.45\columnwidth]{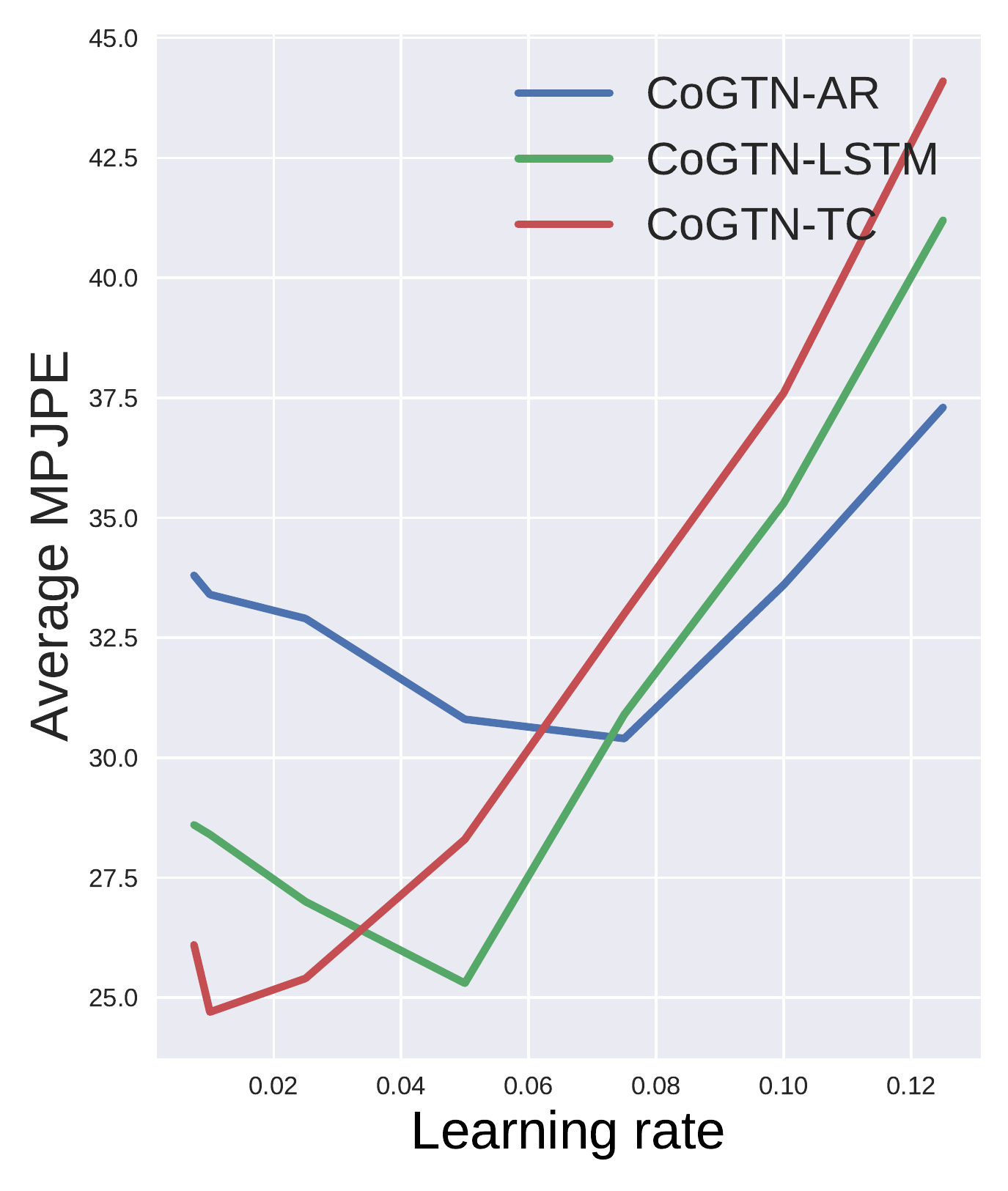}
            \\
           {{\small (a) Regret curves.}} & 
           {{\small (b) Effects of learning rates.}}
    \end{tabular}
    \caption{\small The regret curves and the performance with various learning rates of CoGNNs for online human motion prediction.}
    \label{fig:RegretPlot}
    \end{center}
\end{figure}
According to our theoretical regret analysis, the static regrets converge with ${\rm R}_T = \mathcal{O}(1/\sqrt{T})$.
Here we compute and visualize the static regrets of different forms of CoGNNs to verify our analysis.
For online motion prediction on H3.6M, we illustrate the regret curves along time within $4500$ iterations; see Fig.~\ref{fig:RegretPlot} (a).
The regrets of three forms of CoGNNs could converge to very low values within approximately $2500$ iterations. Additionally, the regret curves reflect that CoGNNs quickly close the performance gaps with the models that only consider the best collaborative pairs, showing the effectiveness and interpretation of our CoPUs.

\mypar{Effect of learning rate $\eta$}
The learning rate affects the training of the collaborative predictor and the collaborative graph in each CoPU. Here we vary the learning rate $\eta$ from $0.0075$ to $0.125$ to test the model performance for online motion prediction on H3.6M. The mean MPJPEs of different forms of CoGNNs are illustrated in Fig.~\ref{fig:RegretPlot} (b). We see that, CoGNN-AR, CoGNN-LSTM and CoGNN-TC achieve their best performance when $\eta$ is $0.075$, $0.05$ and $0.01$, respectively. CoGNN-AR tends to be more stable with various $\eta$ than others; CoGNN-LSTM and CoGNN-TC obtain better performance given the reasonable $\eta$.

\mypar{Effect of number of CoPUs in CoGNNs}
\begin{table}[t]
    \centering
    \scriptsize
    \caption{\small The avergae MPJPE of CoGNNs with different numbers of CoPUs for online human motion prediction.}
    \resizebox{1.0\columnwidth}{!}{
    \begin{tabular}{c|cccc}
        \hline
        \hline
        Methods & 1 CoPU & 2 CoPUs & 3 CoPUs & 4 CoPUs \\
        \hline
        CoGNN-AR & 38.16 & {\bf 34.42} & 34.46 & 35.15 \\
        CoGNN-LSTM & 28.48 & {\bf 26.73} & 27.47 & 26.92 \\
        CoGNN-TC & 26.07 & {\bf 24.39} & 24.74 & 25.03 \\
        \hline
        \hline
    \end{tabular}}
    \label{tab:layer}
\end{table}
Here we analyze how different numbers of CoPUs in the CoGNN architectures affect the performance. We test different forms of CoGNNs with $1$ to $4$ CoPUs for online human motion prediction on H3.6M. The mean MPJPEs are presented in Table~\ref{tab:layer}. For three forms of CoGNNs, two CoPUs help to achieve the lowest MPJPEs, which outperform the models with only one CoPU to different extents, showing the effectiveness of the progressively refinement. For $3$ or $4$ CoPUs, the model performance shows sub-par results due to the slight over-smoothing, but it still converges to stable values.

\mypar{Visualization of the learned collaborative graphs}
To show that our model effectively captures the implicit systematic dependencies by learning collaborative graphs, here we visualize the learned collaborative graphs on the various tasks of online simulated trajectory prediction, human motion prediction and traffic speed prediction, respectively.

We first visualize the collaborative graphs for online simulated trajectory prediction. We compare the learned collaborative graphs and the predefined interactions that are used to simulate the particle interaction systems. We illustrate the adjacency matrices of both simulated graphs and learned graphs. To reflect that our model could adapt to the changing states of the interaction systems, we show the learned graphs at different iteration steps during the online learning process. The adjacency matrices are illustrated in Fig.~\ref{fig:GraphPlot_traj}. We see that, at one iteration step, $t_1$ or $t_2$, the learned collaborative graph captures similar topologies to the simulated graphs with weighted edges to different extents. At different iteration steps, the collaborative graphs could adapt to the changing simulated interaction structures. Note that the element values of the graph adjacency matrices are not similar due to the normalization in our weight update algorithm, while we focus more on the visual similarity of topologies.
In this way, we verify that the collaborative graph could reasonably depict the implicit interactions in the complex systems.

\begin{figure}[tb]
    \centering
    \includegraphics[width=0.96\columnwidth]{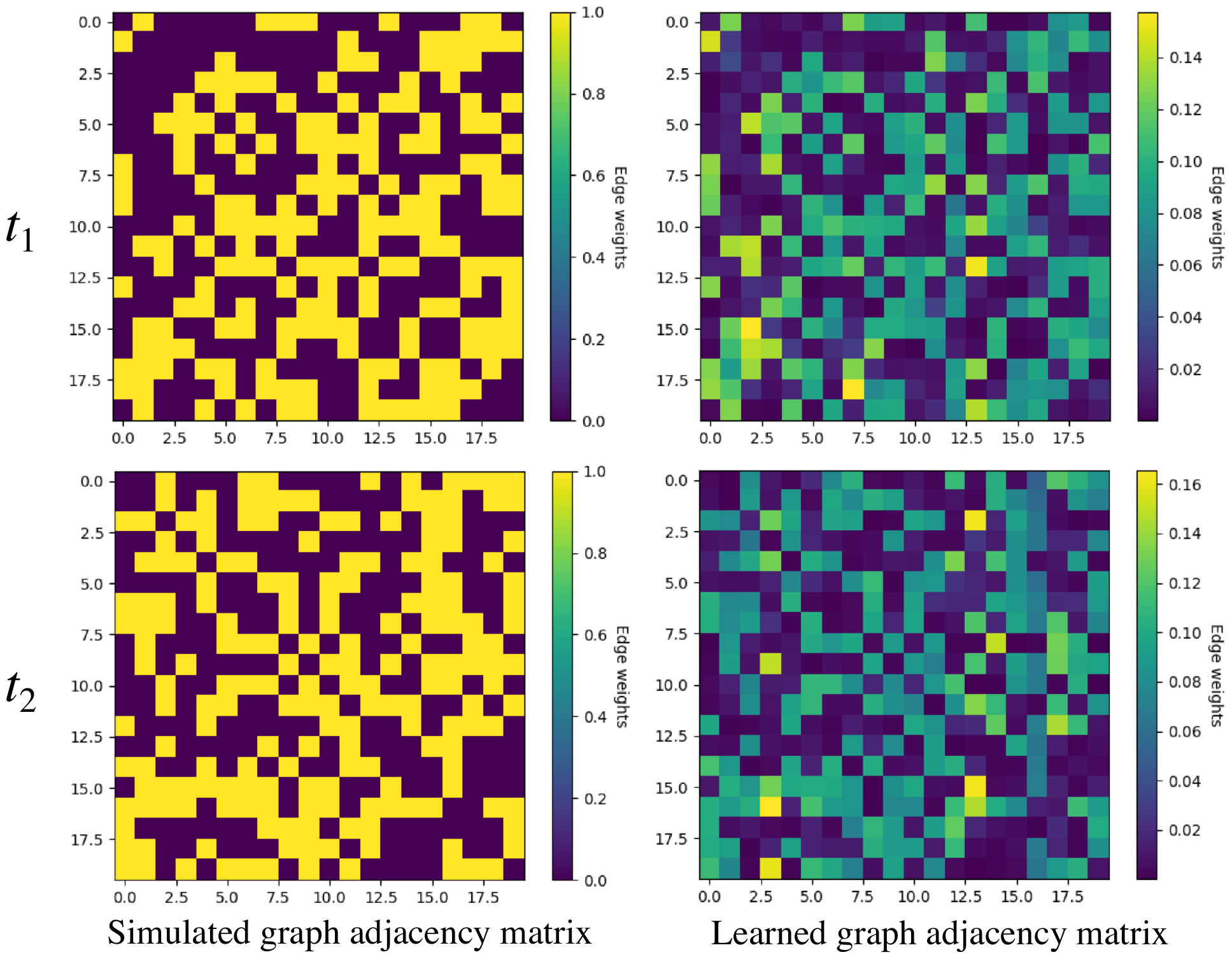}
    \caption{\small Visualization of the adjacency matrices of the collaborative graph on the task of online simulated trajectory prediction. The two rows denotes the different interactions at different time, while the first and the second columns denotes the adjacency matrices of the simulated graphs and the learned graphs, respectively.}
    \label{fig:GraphPlot_traj}
\end{figure}

\begin{figure}[tb]
    \centering
    \includegraphics[width=0.86\columnwidth]{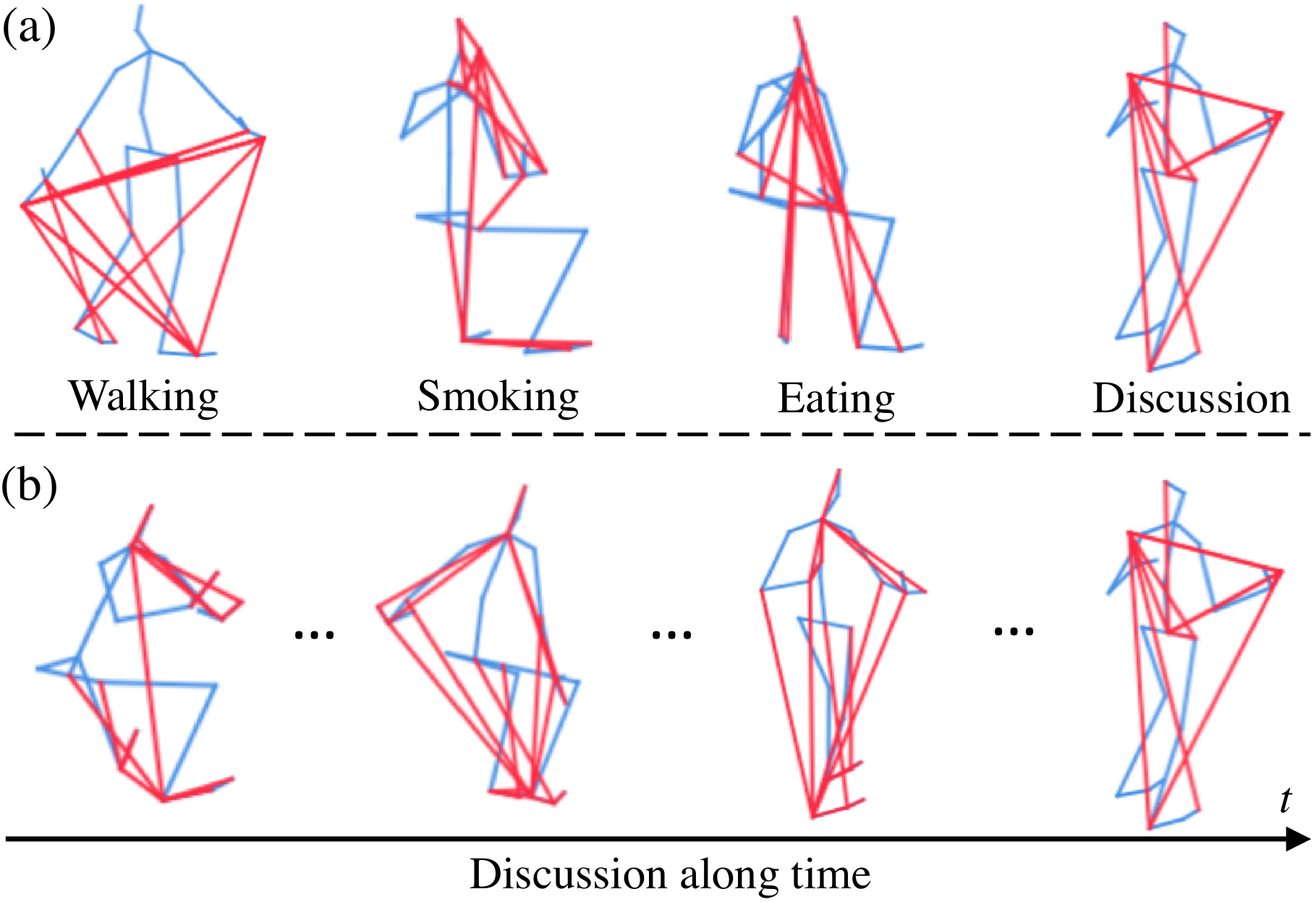}
    \caption{\small Visualization of the collaborative graphs on the task of online human motion prediction. The edges whose weights are larger than $0.2$ are plotted in red. (a) The learned graphs for different actions, showing distinct interactions. (b) The graphs evolves along time, showing dynamic interactions.}
    \label{fig:GraphPlot}
\end{figure}
Here we visualize the learned collaborative graphs on H3.6M dataset. We use red segments to plot the graph edges whose weights are larger than $0.2$ on human poses. Since any edge weights range from $0$ to $1$, and most of them are lower than $0.2$, thus the plotted edges show the important collaborative pairs, i.e., the important influence between two agents. The graphs are illustrated in Fig.~\ref{fig:GraphPlot}.
In plot (a), different actions have different graphs, indicating distinct relations on specific actions; for example, arms and legs affect across left and right for walking, while hands are highly related to other joints for eating and smoking.
In plot (b), we show the dynamic graphs changing along time during discussion.

\begin{figure}[tbh]
    \centering
    \includegraphics[width=0.98\columnwidth]{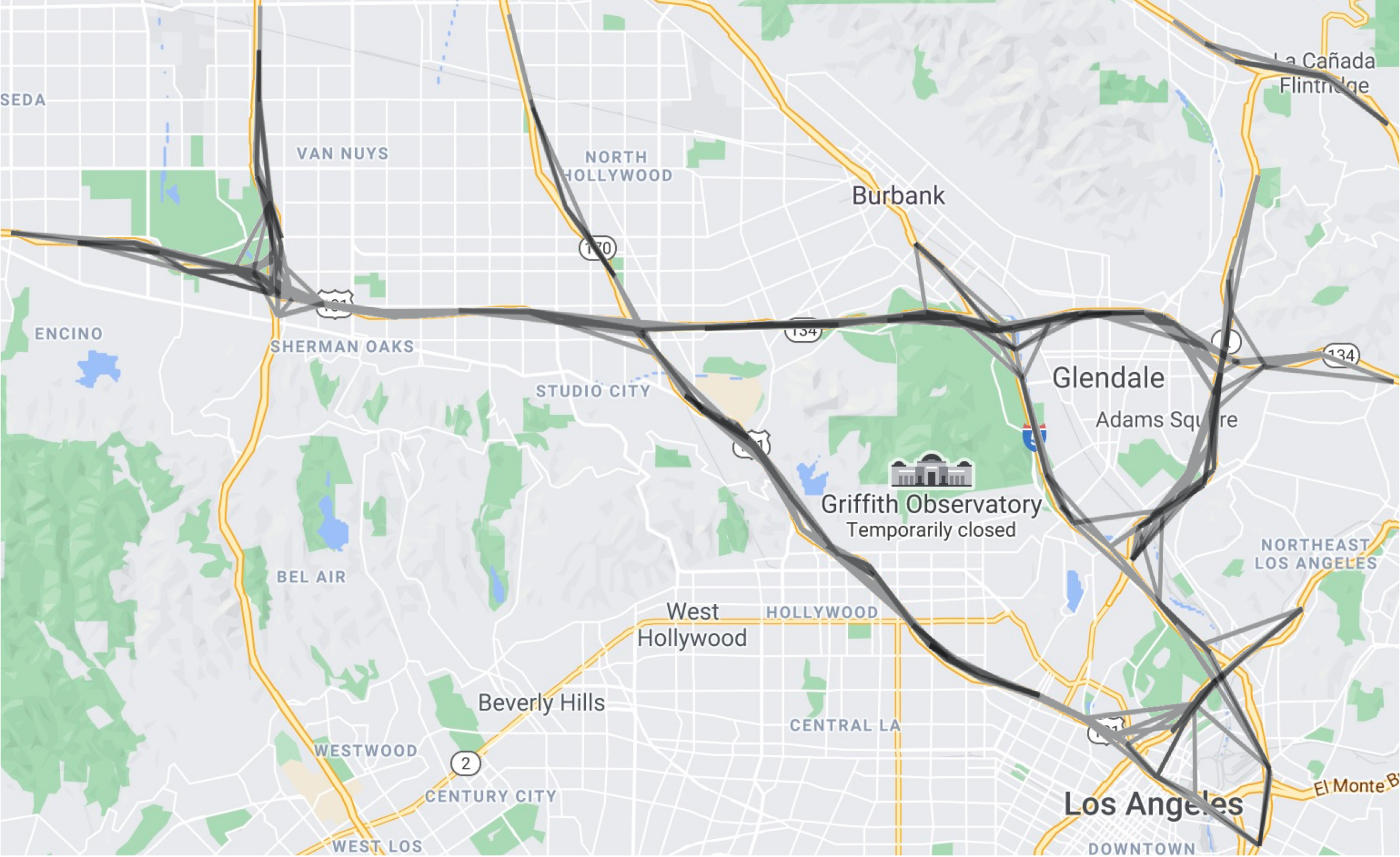}
    \caption{\small The learned collaborative graph for online traffic speed prediction on METR-LA dataset. The edges whose collaborative weights are larger than $0.1$ are illustrated, reflecting the most important collaborative effects among the $207$ nodes.}
    \label{fig:TrafficGraphPlot}
\end{figure}

We also show the learned collaborative graph for online traffic speed prediction. As an example, we show the learned graph on METR-LA dataset. We illustrate the edges whose collaborative weights are larger than $0.1$ in Fig.~\ref{fig:TrafficGraphPlot}, which reflects the most important collaborative effects among the $207$ nodes to some extent.
We see that, the traffic speed data collected by different sensors are mainly affected by their geographically nearby traffic. For example, a certain node on a single long high-way is mainly affected by other nodes on the same high-way, because a single long high-way could be regarded as an isolated systems, in which the the vehicles move and stop collaboratively. Moreover, near the intersections or the traffic circles, one node tends to be affected by the traffic around across roads, since many vehicles perform agglomeration effects in these local regions. Fig.~\ref{fig:TrafficGraphPlot} shows a reasonable structure to depict the collaborative effects.
The collaborative graphs for online human motion prediction have been illustrated in Figure 3 of our main text.
Different from the human motion data, in a traffic scenario, the interactions and relations between two nodes with a relatively long distance is usually weak. That is because traffic data has very strong spatial location attributes, and higher synergy tends to appear in more local areas, such as local congestion during peak hours; and two nodes farther apart are difficult to affect each other because of the decay of collaborative effects along distances.

\section{Conclusion}
\label{sec:conclusion}
We develop a novel method to predict the future statuses of a multi-agent system online. We propose a novel CoPU, which uses a collaborative graph to aggregate multiple collaborative predictors that learn dynamics from collaborative pairs. The collaborative graph depicts the influence level of collaborative pairs, which are adjusted with the guidance from an explicit objective. The regret analysis show that our method achieves the similar performance with the best single collaborative predictor. 
Multiple CoPUs are stacked as a CoGNN.
Experiments demonstrate the effectiveness of our method for various multi-agent forecasting tasks.








%

\small
\bibliography{IEEEabrv}

\begin{thebibliography}{10}
\providecommand{\url}[1]{#1}
\csname url@samestyle\endcsname
\providecommand{\newblock}{\relax}
\providecommand{\bibinfo}[2]{#2}
\providecommand{\BIBentrySTDinterwordspacing}{\spaceskip=0pt\relax}
\providecommand{\BIBentryALTinterwordstretchfactor}{4}
\providecommand{\BIBentryALTinterwordspacing}{\spaceskip=\fontdimen2\font plus
\BIBentryALTinterwordstretchfactor\fontdimen3\font minus
  \fontdimen4\font\relax}
\providecommand{\BIBforeignlanguage}[2]{{%
\expandafter\ifx\csname l@#1\endcsname\relax
\typeout{** WARNING: IEEEtran.bst: No hyphenation pattern has been}%
\typeout{** loaded for the language `#1'. Using the pattern for}%
\typeout{** the default language instead.}%
\else
\language=\csname l@#1\endcsname
\fi
#2}}
\providecommand{\BIBdecl}{\relax}
\BIBdecl

\bibitem{Jahangiri_2015_TITS_AML}
A.~{Jahangiri} and H.~A. {Rakha}, ``Applying machine learning techniques to
  transportation mode recognition using mobile phone sensor data,'' \emph{IEEE
  Transactions on Intelligent Transportation Systems (TITS)}, vol.~16, no.~5,
  pp. 2406--2417, 2015.

\bibitem{Wojtusiak_2012_CMA_MLA}
J.~Wojtusiak, T.~Warden, and O.~Herzog, ``Machine learning in agent-based
  stochastic simulation: Inferential theory and evaluation in transportation
  logistics,'' \emph{Computers \& Mathematics with Applications}, vol.~64,
  no.~12, pp. 3658 -- 3665, 2012.

\bibitem{Yu_2015_ACMSIGAPP_TSM}
W.~Yu, D.~An, D.~Griffith, Q.~Yang, and G.~Xu, ``Towards statistical modeling
  and machine learning based energy usage forecasting in smart grid,''
  \emph{ACM SIGAPP Applied Computing Review}, vol.~15, pp. 6--16, 2015.

\bibitem{Song_2020_AAAI_STS}
C.~Song, Y.~Lin, S.~Guo, and H.~Wan, ``Spatial-temporal synchronous graph
  convolutional networks: A new framework for spatial-temporal network data
  forecasting,'' in \emph{Proceedings of the AAAI Conference on Artificial
  Intelligence (AAAI)}, 2020, pp. 914--921.

\bibitem{Kipf_2018_ICML_NRI}
T.~Kipf, E.~Fetaya, K.-C. Wang, M.~Welling, and R.~Zemel, ``Neural relational
  inference for interacting systems,'' in \emph{Proceedings of International
  Conference on Machine Learning (ICML)}, 2018, pp. 2688--2697.

\bibitem{Graber_2020_CVPR}
C.~Graber and A.~G. Schwing, ``Dynamic neural relational inference,'' in
  \emph{IEEE/CVF Conference on Computer Vision and Pattern Recognition (CVPR)},
  2020, pp. 8513--8522.

\bibitem{NEURIPS2020_e4d8163c}
J.~Li, F.~Yang, M.~Tomizuka, and C.~Choi, ``Evolvegraph: Multi-agent trajectory
  prediction with dynamic relational reasoning,'' in \emph{Advances in Neural
  Information Processing Systems (NeurIPS)}, 2020, pp. 19\,783--19\,794.

\bibitem{Hoi_2018_arxiv_OLA}
S.~C.~H. Hoi, D.~Sahoo, J.~Lu, and P.~Zhao, ``Online learning: A comprehensive
  survey,'' \emph{CoRR}, vol. abs/1802.02871, 2018.

\bibitem{Zivot2006}
E.~Zivor and J.~Wang, ``Vector autoregressive models for multivariate time
  series,'' \emph{Modeling Financial Time Series with S-PLUS®}, pp. 385--429,
  2006.

\bibitem{Lehrmann_2014_CVPR_ENM}
A.~Lehrmann, P.~Gehler, and S.~Nowozin, ``Efficient nonlinear markov models for
  human motion,'' in \emph{The IEEE Conference on Computer Vision and Pattern
  Recognition (CVPR)}, 2014, pp. 1314--1321.

\bibitem{Vishwakarma_1970_IJSS_POE}
K.~P. Vishwakarma, ``Prediction of economic time-series by means of the kalman
  filter,'' \emph{International Journal of Systems Science}, 1970.

\bibitem{Hochreiter_1997_NeurComput_LSTM}
S.~Hochreiter and J.~Schmidhuber, ``Long short-term memory,'' \emph{Neural
  computation}, vol.~9, no.~8, pp. 1735--1780, 1997.

\bibitem{Guo_2019_AAAI_ABS}
S.~Guo, Y.~Lin, N.~Feng, C.~Song, and H.~Wan, ``Attention based
  spatial-temporal graph convolutional networks for traffic flow forecasting,''
  in \emph{AAAI Conference on Artificial Intelligence (AAAI)}, 2019, pp.
  922--929.

\bibitem{Yu_2018_IJCAI_STG}
B.~Yu, H.~Yin, and Z.~Zhu, ``Spatio-temporal graph convolutional networks: A
  deep learning framework for traffic forecasting,'' in \emph{Proceedings of
  International Joint Conference on Artificial Intelligence (IJCAI)}, 2018, pp.
  3634--3640.

\bibitem{Mao_2019_ICCV_LTD}
W.~Mao, M.~Liu, M.~Salzmann, and H.~Li, ``Learning trajectory dependencies for
  human motion prediction,'' in \emph{Proceedings of the IEEE/CVF International
  Conference on Computer Vision (ICCV)}, 2019, pp. 9489--9497.

\bibitem{Cui_2020_CVPR_LDR}
Q.~Cui, H.~Sun, and F.~Yang, ``Learning dynamic relationships for 3d human
  motion prediction,'' in \emph{Proceedings of the IEEE/CVF Conference on
  Computer Vision and Pattern Recognition (CVPR)}, 2020, pp. 6519--6527.

\bibitem{Li_2020_CVPR_DMG}
M.~Li, S.~Chen, Y.~Zhao, Y.~Zhang, Y.~Wang, and Q.~Tian, ``Dynamic multiscale
  graph neural networks for 3d skeleton based human motion prediction,'' in
  \emph{Proceedings of the IEEE/CVF Conference on Computer Vision and Pattern
  Recognition (CVPR)}, 2020, pp. 214--223.

\bibitem{Li_2018_ICLR_Diffusion}
Y.~Li, R.~Yu, C.~Shahabi, and Y.~Liu, ``Diffusion convolutional recurrent
  neural network: Data-driven traffic forecasting,'' in \emph{International
  Conference on Learning Representations (ICLR)}, 2018.

\bibitem{Li_2019_CVPR_ASG}
M.~Li, S.~Chen, X.~Chen, Y.~Zhang, Y.~Wang, and Q.~Tian, ``Actional-structural
  graph convolutional networks for skeleton-based action recognition,'' in
  \emph{Proceedings of the IEEE/CVF Conference on Computer Vision and Pattern
  Recognition (CVPR)}, 2019, pp. 3595--3603.

\bibitem{Hu_2020_CVPR_CMP}
Y.~Hu, S.~Chen, Y.~Zhang, and X.~Gu, ``Collaborative motion prediction via
  neural motion message passing,'' in \emph{Proceedings of the IEEE/CVF
  Conference on Computer Vision and Pattern Recognition (CVPR)}, 2020, pp.
  6319--6328.

\bibitem{Shi_2019_CVPR_TSA}
L.~Shi, Y.~Zhang, J.~Cheng, and H.~Lu, ``Two-stream adaptive graph
  convolutional networks for skeleton-based action recognition,'' in
  \emph{Proceedings of the IEEE/CVF Conference on Computer Vision and Pattern
  Recognition (CVPR)}, 2019, pp. 12\,026--12\,035.

\bibitem{Williams_2003_JTE_MFV}
B.~Williams and L.~Hoel, ``Modeling and forecasting vehicular traffic flow as a
  seasonal arima process: Theoretical basis and empirical results,''
  \emph{Journal of Transportation Engineering}, vol. 129, no.~6, pp. 664--672,
  2003.

\bibitem{Wang_2006_NIPS_GPD}
J.~Wang, A.~Hertzmann, and D.~Fleet, ``Gaussian process dynamical models,'' in
  \emph{Advances in Neural Information Processing Systems (NeurIPS)}, 2006, pp.
  1441--1448.

\bibitem{Taylor_2019_ICML_FCR}
G.~Taylor and G.~Hinton, ``Factored conditional restricted {Boltzmann} machines
  for modeling motion style,'' in \emph{International Conference on Machine
  Learning (ICML)}, 2009, pp. 1025--1032.

\bibitem{Rangapuram_2018_NIPS_DSS}
S.~S. Rangapuram, M.~W. Seeger, J.~Gasthaus, L.~Stella, Y.~Wang, and
  T.~Januschowski, ``Deep state space models for time series forecasting,'' in
  \emph{Advances in Neural Information Processing Systems (NeurIPS)}, 2018, pp.
  7785--7794.

\bibitem{Martinez_2017_CVPR_OHM}
J.~Martinez, M.~Black, and J.~Romero, ``On human motion prediction using
  recurrent neural networks,'' in \emph{The IEEE Conference on Computer Vision
  and Pattern Recognition (CVPR)}, 2017, pp. 4674--4683.

\bibitem{Gui_2018_ECCV}
L.-Y. Gui, Y.-X. Wang, X.~Liang, and J.~M.~F. Moura, ``Adversarial
  geometry-aware human motion prediction,'' in \emph{The European Conference on
  Computer Vision (ECCV)}, 2018, pp. 786--803.

\bibitem{Lai_2018_SIGIR_MLS}
G.~Lai, W.-C. Chang, Y.~Yang, and H.~Liu, ``Modeling long- and short-term
  temporal patterns with deep neural networks,'' in \emph{International ACM
  SIGIR Conference on Research \& Development in Information Retrieval
  (SIGIR)}, 2018, pp. 95--104.

\bibitem{Bai_2018_arxiv_AEE}
S.~Bai, J.~Z. Kolter, and V.~Koltun, ``An empirical evaluation of generic
  convolutional and recurrent networks for sequence modeling,'' \emph{CoRR},
  vol. abs/1803.01271, 2018.

\bibitem{Jain_2016_CVPR_SRD}
A.~Jain, A.~Zamir, S.~Savarese, and A.~Saxena, ``Structural-rnn: Deep learning
  on spatio-temporal graphs,'' in \emph{The IEEE Conference on Computer Vision
  and Pattern Recognition (CVPR)}, 2016, pp. 5308--5317.

\bibitem{Wu_2019_IJCAI_GWD}
Z.~Wu, S.~Pan, G.~Long, J.~Jiang, and C.~Zhang, ``Graph wavenet for deep
  spatial-temporal graph modeling,'' in \emph{Proceedings of International
  Joint Conference on Artificial Intelligence (IJCAI)}, 2019, pp. 1907--1913.

\bibitem{Li_2020_NIPS_GGN}
M.~Li, S.~Chen, Y.~Zhang, and I.~Tsang, ``Graph cross networks with vertex
  infomax pooling,'' in \emph{Advances in Neural Information Processing Systems
  (NeurIPS)}, 2020.

\bibitem{Dobson_2003_JMB_DES}
P.~D. Dobson and A.~J. Doig, ``Distinguishing enzyme structures from
  non-enzymes without alignments,'' \emph{Journal of Molecular Biology (JMB)},
  vol. 330, no.~4, pp. 771--783, 2003.

\bibitem{Kipf_2017_ICLR_SSC}
T.~Kipf and M.~Welling, ``Semi-supervised classification with graph
  convolutional networks,'' in \emph{International Conference on Learning
  Representations (ICLR)}, 2017.

\bibitem{Velickovic_2018_ICLR_GAN}
P.~Veličković, G.~Cucurull, A.~Casanova, A.~Romero, P.~Liò, and Y.~Bengio,
  ``Graph attention networks,'' in \emph{International Conference on Learning
  Representations (ICLR)}, 2018.

\bibitem{9069948}
S.~Liu, M.~Sun, L.~Feng, H.~Qiao, S.~Chen, and Y.~Liu, ``Social neighborhood
  graph and multigraph fusion ranking for multifeature image retrieval,''
  \emph{IEEE Transactions on Neural Networks and Learning Systems (TNNLS)},
  vol.~32, no.~3, pp. 1389--1399, 2021.

\bibitem{9064693}
L.~C.~B. Torres, C.~L. Castro, F.~Coelho, and A.~P. Braga, ``Large margin
  gaussian mixture classifier with a gabriel graph geometric representation of
  data set structure,'' \emph{IEEE Transactions on Neural Networks and Learning
  Systems (TNNLS)}, vol.~32, no.~3, pp. 1400--1406, 2021.

\bibitem{Bruna_2014_ICLR_SNL}
J.~Bruna, W.~Zaremba, A.~Szlam, and Y.~LeCun, ``Spectral networks and locally
  connected networks on graphs,'' in \emph{International Conference on Learning
  Representations (ICLR)}, 2014.

\bibitem{Defferrard_2019_NIPS_CNN}
M.~Defferrard, X.~Bresson, and P.~Vandergheynst, ``Convolutional neural
  networks on graphs with fast localized spectral filtering,'' in
  \emph{Advances in Neural Information Processing Systems (NeurIPS)}, 2016, pp.
  3844--3852.

\bibitem{Hamilton_2017_NIPS_IRL}
W.~Hamilton, Z.~Ying, and J.~Leskovec, ``Inductive representation learning on
  large graphs,'' in \emph{Advances in Neural Information Processing Systems
  (NeurIPS)}, 2017, pp. 1024--1034.

\bibitem{Niepert_2016_ICML_LCN}
M.~Niepert, M.~Ahmed, and K.~Kutzkovl, ``Learning convolutional neural networks
  for graphs,'' in \emph{Proceedings of the International Conference on Machine
  Learning (ICML)}, 2016.

\bibitem{Li_2016_ICLR_GGS}
Y.~Li, D.~Tarlow, M.~Brockschmidt, and R.~Zemel, ``Gated graph sequence neural
  networks,'' in \emph{International Conference on Learning Representations
  (ICLR)}, 2016.

\bibitem{Dai_2016_ICML_DEL}
H.~Dai, B.~Dai, and L.~Song, ``Discriminative embeddings of latent variable
  models for structured data,'' in \emph{Proceedings of the International
  Conference on Machine Learning (ICML)}, 2016.

\bibitem{Xu_2019_ICLR_HPA}
K.~Xu, W.~Hu, J.~Leskovec, and S.~Jegelka, ``How powerful are graph neural
  networks?'' in \emph{International Conference on Learning Representations
  (ICLR)}, 2019.

\bibitem{9079208}
X.~Xu, T.~Wang, Y.~Yang, A.~Hanjalic, and H.~T. Shen, ``Radial graph
  convolutional network for visual question generation,'' \emph{IEEE
  Transactions on Neural Networks and Learning Systems}, vol.~32, no.~4, pp.
  1654--1667, 2021.

\bibitem{Thrun_1998_L2L_LLA}
S.~Thrun, ``Lifelong learning algorithms,'' \emph{Learning to Learn}, pp.
  181--209, 1998.

\bibitem{Lu_2018_TOIS_SPA}
J.~Lu, D.~Sahoo, P.~Zhao, and S.~C.~H. Hoi, ``Sparse passive-aggressive
  learning for bounded online kernel methods,'' \emph{ACM transactions on
  intelligent systems}, vol.~9, no.~4, pp. 216--242, 2018.

\bibitem{Kivinen_2004_TSP_OLK}
J.~{Kivinen}, A.~J. {Smola}, and R.~C. {Williamson}, ``Online learning with
  kernels,'' \emph{IEEE Transactions on Signal Processing}, vol.~52, no.~8, pp.
  2165--2176, 2004.

\bibitem{Dekel_2008_SIAM_TFA}
O.~Dekel, S.~Shalev-Shwartz, and Y.~Singer, ``The forgetron: A kernel-based
  perceptron on a budget,'' \emph{SIAM Journal on Computing}, vol.~37, no.~5,
  p. 1342–1372, 2008.

\bibitem{Wang_2012_JMLR_BCK}
Z.~Wang, K.~Crammer, and S.~Vucetic, ``Breaking the curse of kernelization:
  Budgeted stochastic gradient descent for large-scale svm training,''
  \emph{Journal of Machine Learning Research}, vol.~13, no. 100, pp.
  3103--3131, 2012.

\bibitem{9108601}
J.~Zhang, H.~Ning, X.~Jing, and T.~Tian, ``Online kernel learning with adaptive
  bandwidth by optimal control approach,'' \emph{IEEE Transactions on Neural
  Networks and Learning Systems (TNNLS)}, vol.~32, no.~5, pp. 1920--1934, 2021.

\bibitem{Lu_2016_JMLR_LSOK}
J.~Lu, S.~C. Hoi, J.~Wang, P.~Zhao, and Z.-Y. Liu, ``Large scale online kernel
  learning,'' \emph{Journal of Machine Learning Research}, vol.~17, no.~47, pp.
  1--43, 2016.

\bibitem{Bouboulis_2018_TSP_ODL}
P.~{Bouboulis}, S.~{Chouvardas}, and S.~{Theodoridis}, ``Online distributed
  learning over networks in rkh spaces using random fourier features,''
  \emph{IEEE Transactions on Signal Processing}, vol.~66, no.~7, pp.
  1920--1932, 2018.

\bibitem{Ding_2017_ICDM_LSK}
Y.~{Ding}, C.~{Liu}, P.~{Zhao}, and S.~C.~H. {Hoi}, ``Large scale kernel
  methods for online auc maximization,'' in \emph{IEEE International Conference
  on Data Mining (ICDM)}, 2017, pp. 91--100.

\bibitem{Shen_2019_JMLR_RFO}
Y.~Shen, T.~Chen, and G.~B. Giannakis, ``Random feature-based online
  multi-kernel learning in environments with unknown dynamics,'' \emph{Journal
  of Machine Learning Research}, vol.~20, no.~22, pp. 1--36, 2019.

\bibitem{Xiu_2018_BMVC_PFE}
Y.~Xiu, J.~Li, H.~Wang, Y.~Fang, and C.~Lu, ``Pose flow: Efficient online pose
  tracking,'' in \emph{British Machine Vision Conference (BMVC)}, 2018, pp.
  1--12.

\bibitem{Butepage_2018_ICRA_AMF}
J.~{Bütepage}, H.~{Kjellström}, and D.~{Kragic}, ``Anticipating many futures:
  Online human motion prediction and generation for human-robot interaction,''
  in \emph{International Conference on Robotics and Automation (ICRA)}, 2018,
  pp. 4563--4570.

\bibitem{Marchetti_2020_CVPR_MMA}
F.~Marchetti, F.~Becattini, L.~Seidenari, and A.~D. Bimbo, ``Mantra: Memory
  augmented networks for multiple trajectory prediction,'' in \emph{IEEE/CVF
  Conference on Computer Vision and Pattern Recognition (CVPR)}, 2020, pp.
  7143--7152.

\bibitem{Cai_2020_ECCV_LPJ}
Y.~Cai, L.~Huang, Y.~Wang, T.-J. Cham, J.~Cai, J.~Yuan, J.~Liu, X.~Yang,
  Y.~Zhu, X.~Shen, D.~Liu, J.~Liu, and N.~M. Thalmann, ``Learning progressive
  joint propagation for human motion prediction,'' in \emph{European Conference
  on Computer Vision (ECCV)}, 2020, pp. 226--242.

\bibitem{Hazan_2016_FTM_IOC}
E.~Hazan, ``Introduction to online convex optimization,'' \emph{Foundations and
  Trends in Machine Learning}, vol.~2, no. 3-4, pp. 157--325, 2016.

\bibitem{Ionescu_2014_H36m}
C.~Ionescu, D.~Papava, V.~Olaru, and C.~Sminchisescu, ``Human3.6m: Large scale
  datasets and predictive methods for 3d human sensing in natural
  environments,'' \emph{IEEE Transactions on Pattern Analysis and Machine
  Intelligence}, vol.~36, no.~7, pp. 1325--1339, 2014.

\bibitem{Li_2018_CVPR_CSS}
C.~Li, Z.~Zhang, W.~Sun~Lee, and G.~Hee~Lee, ``Convolutional sequence to
  sequence model for human dynamics,'' in \emph{The IEEE Conference on Computer
  Vision and Pattern Recognition (CVPR)}, 2018, pp. 5226--5234.

\bibitem{Mao_2020_ECCV_HRI}
W.~Mao, M.~Liu, and M.~Salzmann, ``History repeats itself: Human motion
  prediction via motion attention,'' in \emph{European Conference on Computer
  Vision (ECCV)}, 2020, pp. 474--489.

\bibitem{Bai_2020_NIPS_AGC}
L.~Bai, L.~Yao, C.~Li, X.~Wang, and C.~Wang, ``Adaptive graph convolutional
  recurrent network for traffic forecasting,'' in \emph{Advances in Neural
  Information Processing Systems (NeurIPS)}, 2020.

\bibitem{Jagadish_2014_Commu_BDI}
H.~V. Jagadish, J.~Gehrke, A.~Labrinidis, Y.~Papakonstantinou, J.~M. Patel,
  R.~Ramakrishnan, and C.~Shahabi, ``Big data and its technical challenges,''
  \emph{Commun. ACM}, vol.~57, no.~7, p. 86–94, 2014.

\bibitem{Chen_2011_TRR_FPM}
C.~Chen, K.~Petty, and A.~Skabardonis, ``Freeway performance measurement
  system: Mining loop detector data,'' \emph{Transportation Research Record},
  vol. 1748, pp. 96--102, 2001.

\end{thebibliography}
\bibliographystyle{IEEEtran}

%

\begin{IEEEbiography}
[{\includegraphics[width=1in,height=1.25in,clip,keepaspectratio]{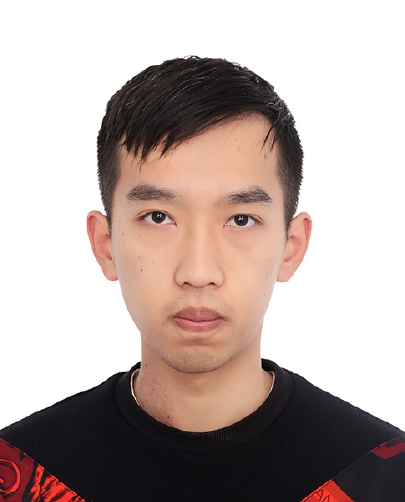}}]
{Maosen Li} recieved the B.E. degree in optical engineering from University of Electronic Science and Technology of China (UESTC), Chengdu, China, in 2017. He is working toward the Ph.D. degree at Cooperative Meidianet Innovation Center in Shanghai Jiao Tong University since 2017. His research interests include computer vision, machine learning, graph representation learning, and video analysis. He is the reviewer of some prestigious international journals and conferences, including IEEE-TPAMI, IEEE-TNNLS, IJCV, IEEE-TMM, PR, ICML, NeurIPS and AAAI. He is a student member of the IEEE.
\end{IEEEbiography}

\begin{IEEEbiography}
[{\includegraphics[width=1in,height=1.25in,clip,keepaspectratio]{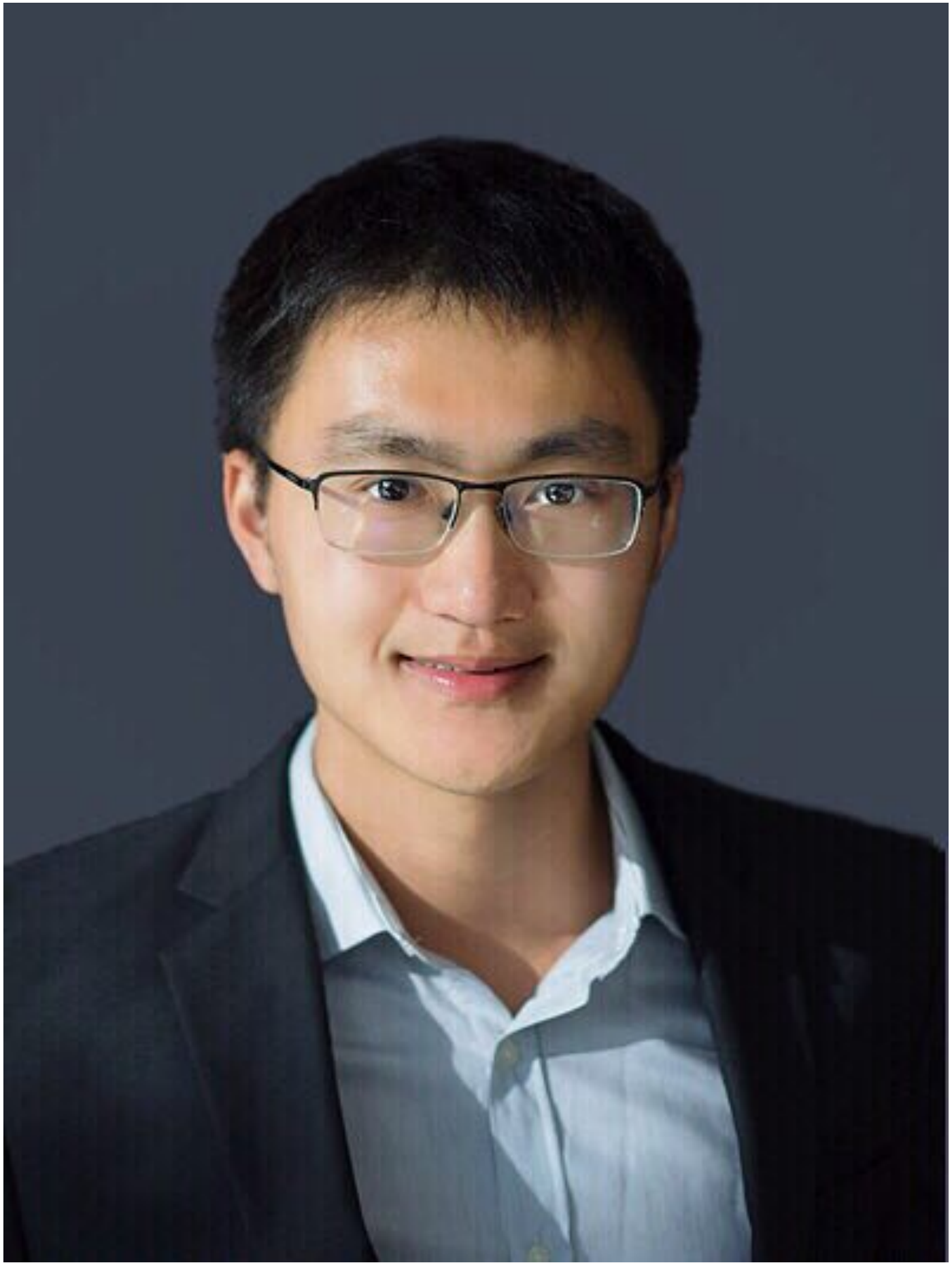}}]
{Siheng Chen} is an associate professor at Shanghai Jiao Tong University. Before that, he was a research scientist at Mitsubishi Electric Research Laboratories (MERL). Before joining MERL, he was an autonomy engineer at Uber Advanced Technologies Group, working on the perception and prediction systems of self-driving cars. Before joining Uber, he was a postdoctoral research associate at Carnegie Mellon University. Chen received the doctorate in Electrical and Computer Engineering from Carnegie Mellon University in 2016, where he also received two masters degrees in Electrical and Computer Engineering and Machine Learning, respectively. He received his bachelor's degree in Electronics Engineering in 2011 from Beijing Institute of Technology, China. Chen was the recipient of the 2018 IEEE Signal Processing Society Young Author Best Paper Award. His coauthored paper received the Best Student Paper Award at IEEE GlobalSIP 2018. He organized the special session "Bridging graph signal processing and graph neural networks" at ICASSP 2020. His research interests include graph signal processing, graph neural networks and autonomous driving. He is a member of IEEE.
\end{IEEEbiography}

\begin{IEEEbiography}
[{\includegraphics[width=1in,height=1.25in,clip,keepaspectratio]{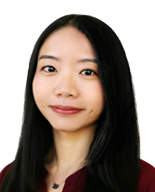}}]
{Yanning Shen} received her Ph.D. degree from the University of Minnesota (UMN) in 2019. She was a finalist for the Best Student Paper Award at the 2017 IEEE International Workshop on Computational Advances in Multi-Sensor Adaptive Processing, and the 2017 Asilomar Conference on Signals, Systems, and Computers. She was selected as a Rising Star in EECS by Stanford University in 2017, and received the UMN Doctoral Dissertation Fellowship in 2018. Her research interests span the areas of machine learning, network science, data science and statistical-signal processing.
\end{IEEEbiography}

\begin{IEEEbiography}
[{\includegraphics[width=1in,height=1.25in,clip,keepaspectratio]{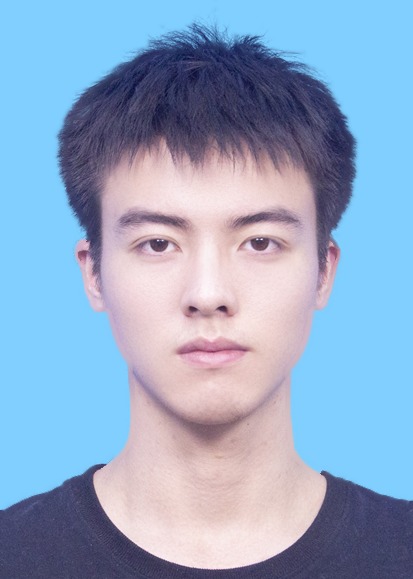}}]
{Genjia Liu} is an undergraduate student at Shanghai Jiao Tong University since 2017. He is going to work toward the Ph.D. degree at Cooperative Medianet Innovation Center in Shanghai Jiao Tong University from 2021. His research interests include graph signal processing and graph representation learning.
\end{IEEEbiography}

\begin{IEEEbiography}
[{\includegraphics[width=1in,height=1.25in,clip,keepaspectratio]{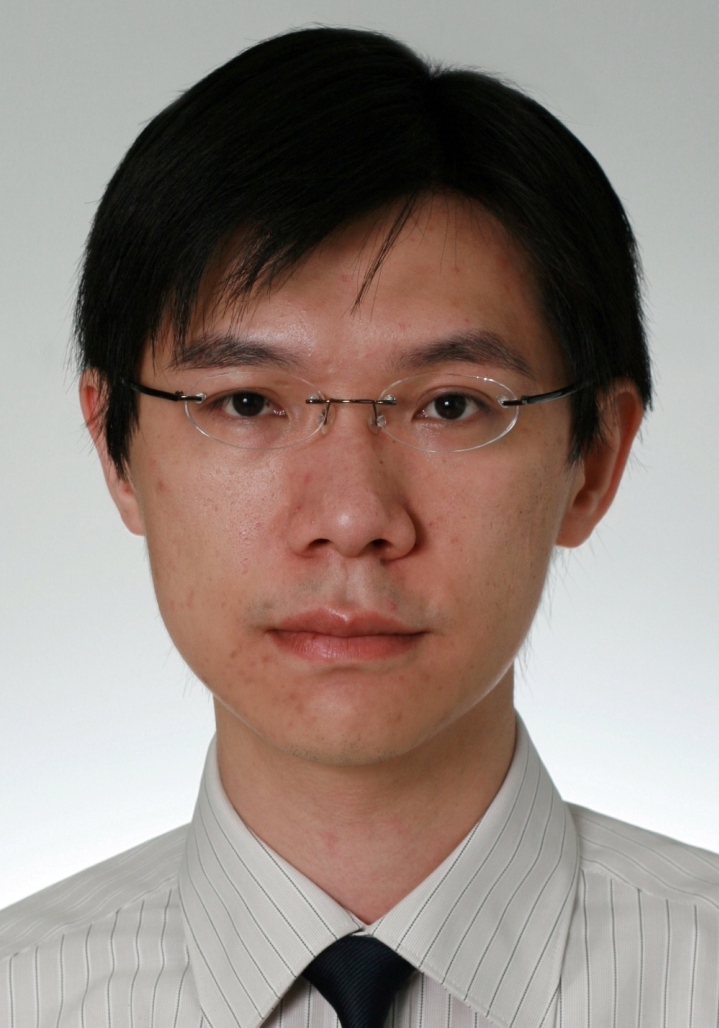}}]
{Ivor Tsang} is Professor of Artificial Intelligence with the University of Technology Sydney. He is also the Research Director of the Australian Artificial Intelligence Institute. His research interests include transfer learning, deep generative models, and big data analytics.  In 2013, Prof Tsang received his prestigious ARC Future Fellowship for his research regarding Machine Learning on Big Data. In 2019, his JMLR paper titled "Towards ultrahigh dimensional feature selection for big data" received the International Consortium of Chinese Mathematicians Best Paper Award. In 2020, Prof Tsang was recognized as the AI 2000 AAAI/IJCAI Most Influential Scholar in Australia for his outstanding contributions to the field of AAAI/IJCAI between 2009 and 2019. His research on transfer learning granted him the Best Student Paper Award at CVPR 2010 and the 2014 IEEE TMM Prize Paper Award. In addition, he received the IEEE TNN Outstanding 2004 Paper Award in 2007. He serves as a Senior Area Chair/Area Chair for NeurIPS, ICML, AISTATS, AAAI and IJCAI,  and the Editorial Board for JMLR, MLJ, and IEEE TPAMI.
\end{IEEEbiography}

\begin{IEEEbiography}
[{\includegraphics[width=1in,height=1.25in,clip,keepaspectratio]{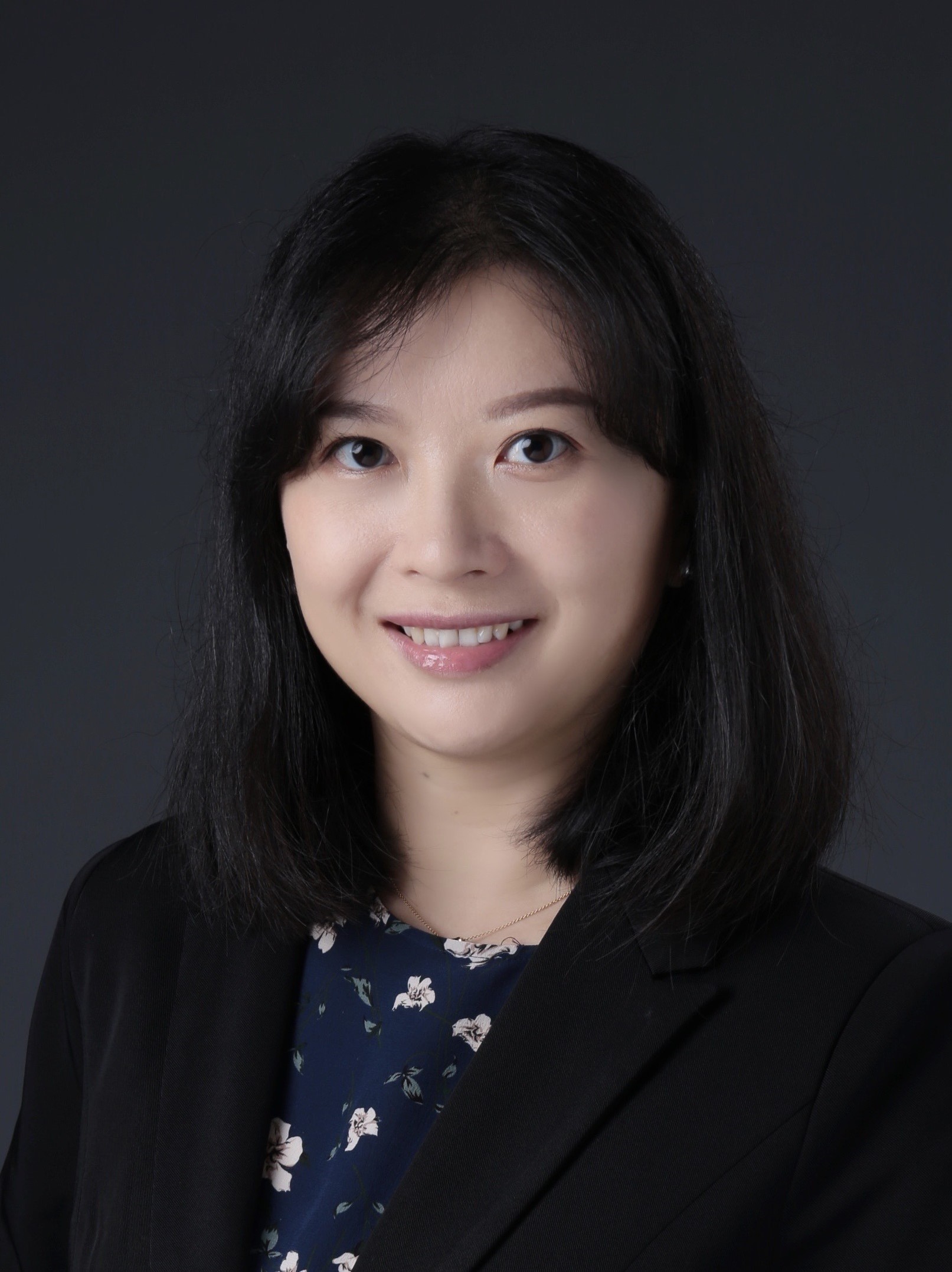}}]
{Ya Zhang}
is currently a professor at the Cooperative Medianet Innovation Center, Shanghai Jiao Tong University. Her research interest is mainly in machine learning with applications to multimedia and healthcare. Dr. Zhang holds a Ph.D. degree in Information Sciences and Technology from Pennsylvania State University and a bachelor's degree from Tsinghua University in China. Before joining Shanghai Jiao Tong University, Dr. Zhang was a research manager at Yahoo! Labs, where she led an R\&D team of researchers with strong backgrounds in data mining and machine learning to improve the web search quality of Yahoo international markets. Prior to joining Yahoo, Dr. Zhang was an assistant professor at the University of Kansas with a research focus on machine learning applications in bioinformatics and information retrieval. Dr. Zhang has published more than 70 refereed papers in prestigious international conferences and journals, including TPAMI, TIP, TNNLS, ICDM, CVPR, ICCV, ECCV, and ECML. She currently holds 5 US patents and 4 Chinese patents and has 9 pending patents in the areas of multimedia analysis. She was appointed the Chief Expert for the 'Research of Key Technologies and Demonstration for Digital Media Self-organizing' project under the 863 program by the Ministry of Science and Technology of China. She is a member of IEEE.
\end{IEEEbiography}

\end{document}